\documentclass[letterpaper]{article} 
\usepackage{aaai23}  
\usepackage{times}  
\usepackage{helvet}  
\usepackage{courier}  
\usepackage[hyphens]{url}  
\usepackage{graphicx} 
\usepackage{natbib}  
\usepackage{caption} 
\usepackage{algorithm}
\usepackage{algorithmic}

\usepackage{newfloat}
\usepackage{listings}
\DeclareCaptionStyle{ruled}{labelfont=normalfont,labelsep=colon,strut=off} 
\floatstyle{ruled}
\newfloat{listing}{tb}{lst}{}
\floatname{listing}{Listing}
\usepackage{amsmath}
\usepackage{subfig}
\usepackage{amsthm}
\usepackage{amsfonts}
\usepackage[]{algorithm}
\usepackage[algo2e]{algorithm2e}
\usepackage[]{algorithmic}

\newcommand{\calL}{\mathcal{L}}

\newcommand{\calD}{\mathcal{D}}

\newtheorem{mydef}{Definition}
\newtheorem{theorem}{Theorem}
\newtheorem{lemma}{Lemma}

\title{\textsc{XRand}: Differentially Private Defense against Explanation-Guided Attacks}

\author {
    Truc Nguyen,\equalcontrib\textsuperscript{\rm 1}
    Phung Lai,\equalcontrib\textsuperscript{\rm 2}
    NhatHai Phan,\textsuperscript{\rm 2}
    My T. Thai \textsuperscript{\rm 1}\thanks{Corresponding author}
}
\affiliations {
    \textsuperscript{\rm 1} University of Florida, Gainesville, FL 32611\\
    \textsuperscript{\rm 2} New Jersey Institute of Technology, Newark, NJ 07102\\
    truc.nguyen@ufl.edu, tl353@njit.edu, phan@njit.edu, mythai@cise.ufl.edu
}

\usepackage{bibentry}

\begin{document}

\maketitle

\begin{abstract}
    Recent development in the field of explainable artificial intelligence (XAI) has helped improve trust in Machine-Learning-as-a-Service (MLaaS) systems, in which an explanation is provided together with the model prediction in response to each query. However, XAI also opens a door for adversaries to gain insights into the black-box models in MLaaS, thereby 
    making the models more vulnerable to several attacks. For example, feature-based explanations (e.g., SHAP) could expose the top important features that a black-box model focuses on. Such disclosure has been exploited to craft effective backdoor triggers against malware classifiers. To address this trade-off, we introduce a new concept of achieving local differential privacy (LDP) in the explanations, and from that we establish a defense, called \textsc{XRand}, against such attacks. We show that our mechanism restricts the information that the adversary can learn about the top important features, while maintaining the faithfulness of the explanations. 
\end{abstract}

\section{Introduction}
Over decades, successes in machine learning (ML) have promoted a strong wave of AI applications that deliver vast benefits to a diverse range of fields. Unfortunately, due to their complexity, ML models suffer from opacity in terms of explainability, which reduces the trust in and the verifiability of the decisions made by the models. To meet the necessity of transparent decision making, model-agnostic explainers have been developed to help create effective, more human-understandable AI systems, such as LIME \citep{ribeiro2016should} and SHAP \citep{lundberg2017unified}, among many others \citep{sundararajan2017axiomatic, selvaraju2017grad, vu2020pgm, vu2022neucept, ying2019gnnexplainer,shrikumar2017learning, vu2021c}.

In MLaaS systems, a customer can build an ML model by uploading their data or crowdsourcing data, and executing an ML training algorithm. Then, the model is deployed in the cloud where users can receive the model predictions for input queries. MLaaS assumes black-box models as the end-users have no knowledge about the algorithm or internal information about the underlying ML models. Several proposals advocate for deploying model explanations in the cloud, such that a predicted label and an explanation are returned for each query to provide transparency for end-users. In practice, such an explainable MLaaS system model has been developed by many cloud providers \cite{minthigpen,ibm}, with applications in both industries \cite{fico} and academic research \cite{shokri2021privacy,milli2019model}. 

Despite the great potential of those explainers to improve the transparency and understanding of ML models in MLaaS, they open a trade-off in terms of security. Specifically, they allow adversaries to gain insights into black-box models, essentially uncovering certain aspects of the models that make them more vulnerable. 
Such an attack vector has recently been exploited by the research community to conduct several \textit{explanation-guided attacks}  \cite{shokri2021privacy,milli2019model,miura2021megex,zhao2021exploiting,severi2021explanation}. It was shown that an explainer may expose the top important features on which a black-box model is focusing, by aggregating over the explanations of multiple samples. An example of utilizing such information is the recent highly effective explanation-guided backdoor attack (XBA) against malware classifiers investigated by  \citep{severi2021explanation}. The authors suggest that SHAP can be used to extract the top-$k$ goodware-oriented features. The attacker then selects a combination of those features and their values 
for crafting a trigger; and injects trigger-embedded goodware samples into the training dataset of a malware classifier, with an aim of changing the prediction of malware samples embedded with the same trigger at inference time. 

\noindent\textbf{Defense Challenges.} To prevent an adversary from exploiting the explanations, we need to control the information leaked through them, especially the top-$k$ features. Since the explanation on each queried sample is returned to the end users, a viable defense is to randomize the explanation such that it is difficult for attackers to distinguish top-$k$ features while maintaining valuable explainability for the decision-making process. A well-applied technique to achieve this is preserving local differential privacy (LDP) \citep{erlingsson2014rappor,wang2017locally} on the explanations. 
However, existing LDP-based approaches \cite{erlingsson2014rappor,xiong2020comprehensive,sun2020ldp,zhao2020local} are not designed to protect the top-$k$ features aggregated over the returned explanations on queried data samples. 
Therefore, optimizing the trade-off between defenses against explanation-guided attacks and model explainability is 
an open problem.



\paragraph{Contributions.} (1) We introduce a new concept of achieving LDP in model explanations that simultaneously protects the top-$k$ features from being exploited by attackers while maintaining the faithfulness of the explanations. Based on this principle, we propose a defense against explanation-guided attacks on MLaaS, called \textsc{XRand}, by devising a novel two-step LDP-preserving mechanism. First, at the aggregated explanation, we incorporate the explanation loss into the randomized probabilities in LDP to make top-$k$ features indistinguishable to the attackers. Second, at the sample-level explanation, guided by the first step, we minimize the explanation loss on each sample while keeping the features at the aggregated explanation intact.
(2) Then, we theoretically analyze the robustness of our defense against the XBA in MLaaS by establishing two certified robustness bounds in both training time and inference time. 
(3) Finally, we evaluate the effectiveness of \textsc{XRand} in mitigating the XBA on cloud-hosted malware classifiers.

\paragraph{Organization.} The remainder of the paper is structured in the following manner. Section \ref{sec:prelim} presents some background knowledge for our paper. Section \ref{sec:mlaas} discusses the explanation-guided attacks against MLaaS and establishes the security model for our defense. Our defense, \textsc{XRand}, is introduced in Section \ref{sec:def} where its certified robustness bounds are presented in Section \ref{sec:rob}. Experimental evaluation of  our solution is given in Section \ref{sec:exp}. Section \ref{sec:rel} discusses related work and Section \ref{sec:con} provides some concluding remarks.

\section{Preliminaries} \label{sec:prelim}


\paragraph{Local explainers.} The goal of model explanations is to capture the importance of each feature value of a given point of interest with respect to the decision made by the classifier and which class it is pushing that decision toward. Given a sample $x \in \mathbb{R}^d$ where $x_j$ denotes the $j^{th}$ feature of the sample, let $f$ be a model function in which $f(x)$ is the probability that $x$ belongs to a certain class. An explanation of the model's output $f(x)$ takes the form of an explanation vector $w_x \in \mathbb{R}^d$ where the $j^{th}$ element of $w_x$ denotes the degree to which the feature $x_j$ influences the model's decision. In general, higher values of $w_{x_j}$ imply a higher impact.



Perturbation-based explainers, such as SHAP \cite{lundberg2017unified}, obtain an explanation vector $w_x$ for $x$ via training a surrogate model of the form $g(x) = w_{x_0} + \sum_{j=1}^d w_{x_j} x_j$ by minimizing a loss function $\calL(f,g)$ that measures how unfaithful $g$ is in approximating $f$. 

\begin{itemize}
    \item \textit{\underline {Sample-level explanation.}} In the context of this paper, we refer to $w_x$ as a sample-level explanation. 
    \item \textit{\underline {Aggregated explanation.}}
We denote an aggregated explanation $w$ as the sum of explanation vectors across samples in a certain set $\mathcal{X}$, i.e., $w_{\mathcal{X}} = \sum_{x\in \mathcal{X}} w_x$. When $\mathcal{X}$ is clear from the context, we shall use a shorter notation $w$.
\end{itemize}

\paragraph{Local Differential Privacy (LDP).} LDP is one of the state-of-the-arts and provable approaches to achieve individual data privacy. LDP-preserving mechanisms \citep{erlingsson2014rappor,wang2017locally,bassily2015local,duchi2018minimax,acharya2019hadamard} generally build on the ideas of randomized response (RR) \citep{warner1965randomized}. 

\begin{mydef}{$\varepsilon$-LDP.} A randomized algorithm $\mathcal{A}$
satisfies $\varepsilon$-LDP, if for any two inputs $x$ and $x'$, and for all possible outputs $\mathcal{O} \in Range(\mathcal{A})$, we have:
$Pr[\mathcal{A}(x) = \mathcal{O}] \leq e^{\varepsilon} Pr[\mathcal{A}(x') = \mathcal{O}]$,
where $\varepsilon$ is a privacy budget and $Range(\mathcal{A})$ denotes every possible output of  $\mathcal{A}$. 
\label{Different Privacy} 
\end{mydef}

The privacy budget $\varepsilon$ controls the amount by which the distributions induced by inputs $x$ and $x'$ may differ. A smaller $\varepsilon$ enforces a stronger privacy guarantee. 

\section{XAI-guided Attack Against MLaaS} \label{sec:mlaas}
We discuss how XAI can be used to gain insights into MLaaS models, and establish the threat model for our work.

\subsection{Exposing MLaaS via XAI}
From a security viewpoint, releasing additional information about a model's mechanism is a perilous prospect. As a function of the model that is trained on a private dataset, an explanation may unintentionally disclose critical information about the training set, more than what is needed to offer a useful interpretation. Moreover, the explanations may also expose the internal mechanism of the black-box models. 
For example, first, the behavior of explanations varies based on whether the query sample was a member of the training dataset, making the model vulnerable to membership inference attacks \citep{shokri2021privacy}. 
Second, the explanations can be coupled with the predictions to improve the performance of generative models which, in turn, strengthens some model inversion attacks \citep{zhao2021exploiting}. Furthermore, releasing the explanations exposes how the black-box model acts upon an input sample, essentially giving up more information about its inner workings for each query, hence, model extractions attacks can be carried out with far fewer queries, as discussed in \citep{milli2019model,miura2021megex}.




Finally, \cite{severi2021explanation} argues that the explanations allow an adversary to gain insight into a model’s decision boundary in a generic, model-agnostic way. The SHAP values can be considered as an approximation of the confidence of the decision boundary along each feature dimension. Hence, features with SHAP values that are near zero infer low-confidence areas of the decision boundary. On the other hand, features with positive SHAP values imply that they strongly contribute to the decision made by the model. As a result, it provides us with an indication of the overall orientation for each feature, thereby exposing how the model rates the importance of each feature.

\subsection{XAI-guided Backdoor Attack against MLaaS}

The XBA on malware classifiers \cite{severi2021explanation} suggests that explanations make the model vulnerable to backdoor attacks, as they reveal the top important features. Thus, it is natural to mount this XBA against a black-box model in MLaaS where an explanation is returned for each user query.



\begin{figure}
    \centering
    \includegraphics[width=0.7\linewidth]{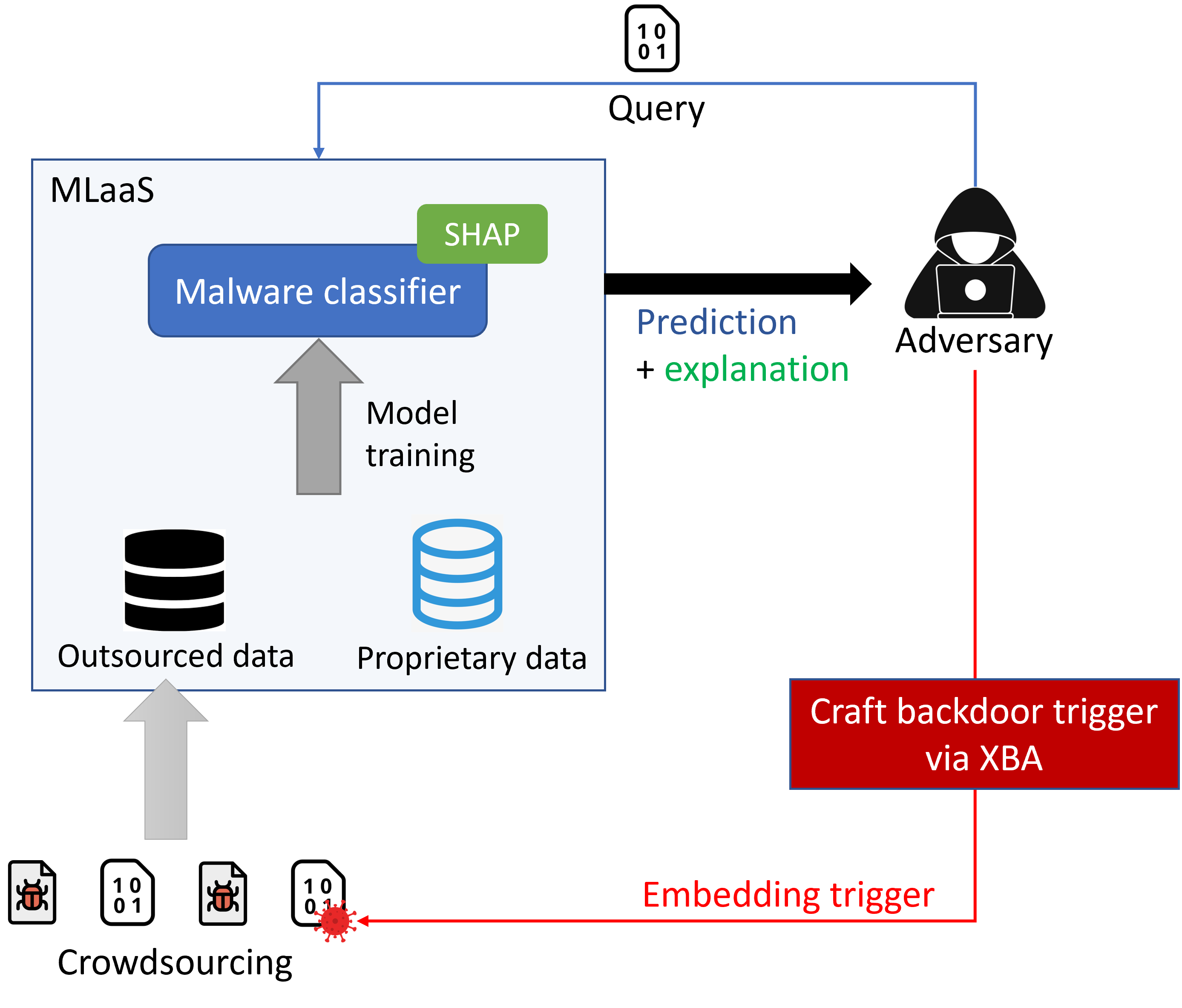}
    \caption{System model of a cloud-hosted malware classifier that leverages crowdsourced data for model training. 
    }
    \label{fig:overview}
\end{figure}

\paragraph{System Model.} Fig. \ref{fig:overview} illustrates the system model for our work. We consider an MLaaS system where a malware classifier is deployed on a cloud-based platform. For training, the system crowdsources threat samples via user-submitted binaries to assemble a set of outsourced data. 
This set of outsourced data is then combined with a set of proprietary data to construct the training data to train the malware classifier. 

We denote $\mathcal{D} = \{ (x_n,y_n) \}_{n=1}^N$ as the set of proprietary training data. The dataset contains sample $x_n\in \mathbb{R}^d$ and its ground-truth label $y_n \in \{0,1\}$, where $y_n = 0$ denotes a goodware sample, and $y_n = 1$ denotes a malware sample. On input $x$, the model $f : \mathbb{R}^d \rightarrow \mathbb{R}$ outputs the score $f(x) \in [0, 1]$. This score is then compared with a threshold of 0.5 to obtain the predicted label for $x$.

During inference time, given a query containing a binary sample $x$, the system returns the predicted label 
with a SHAP explanation $w_x$ for the decision ($w_x\in \mathbb{R}^d$). We consider an adversary who plays the role of a user in this system and can send queries at his discretion. The adversary exploits the returned explanations to craft backdoor triggers that will be injected to the system via the crowdsourcing process, thereby poisoning the outsourced data.

\paragraph{Threat Model.} The attacker's goal is to alter the training procedure by injecting poisoned samples into the training data, generating poisoned training data such that the resulting backdoored classifier differs from a clean classifier. An ideal backdoored classifier has the same response to a clean input as the clean classifier, but it gives an attack-chosen prediction when the input is embedded with the trigger.


Our defense assumes a strong adversary such that he can tamper with the training dataset at his discretion without major constraints. To prevent the adversary from setting arbitrary values for the features in the trigger, the set of values that can be used is limited to the ones that exist in the dataset. This threat model promotes a defense under worst-case scenarios from the perspective of the defender.

\paragraph{Crafting backdoor triggers.} To craft a backdoor trigger in XBA, the adversary tries to obtain the top goodware-oriented feature by querying classifier $f$ with samples $\{x\}_{x\in A}$ from their dataset $A$ and obtaining the SHAP explanation $w_x$ for each of them. The sum of the SHAP values across all queried samples $w_A = \sum_{x\in A} w_x$ approximately indicates the importance of each feature, and whether it is goodware- or malware-oriented. From that, the attacker greedily selects a combination of the most goodware-oriented features to create the trigger \citep{severi2021explanation}.

\section{\textsc{XRand} -- Local DP Defense} \label{sec:def}

This section describes our defense, \textsc{XRand}, a novel two-step explanation-guided randomized response (RR) mechanism.
Our idea is to incorporate the model explainability into the randomization probabilities in \textsc{XRand} to guide the aggregated explanation while minimizing the difference between the perturbed explanation's surrogate model $g'(x)$ and the model $f(x)$ at the sample-level explanation. We call the difference between $g'(x)$ and $f(x)$ an explanation loss $\mathcal{L}$, quantified as follows: 
\begin{equation} \label{loss 2}
    \calL = \sum_{z\in N(x)} \left(g'(z) - f(z)\right)^2 \exp \left(-\frac{\|z-x\|^2}{\sigma^2}\right) 
\end{equation}
where $x$ is an input sample, its neighborhood $N(x)$ is generated by the explainer's sampling method. 
This function captures the difference between the modified explainer's linear surrogate model $g'(x)$ and the model $f(x)$, essentially measuring how unfaithful $g'$ is in approximating $f$.






To defend against explanation-guided attacks that utilize the top-$k$ features of the aggregated explanation (e.g., XBA exploits the top-$k$ goodware-oriented features), our idea is to 
randomly disorder some top-$k$ features 
under LDP guarantees, thereby protecting the privacy of those features. 
This raises the following question: \textit{What top-$k$ features and which data samples should be randomized to optimize the explainability of data samples while guaranteeing that the attackers cannot infer the top-$k$ features?}



\begin{algorithm}[t] 
\caption{\textsc{XRand}: Explanation-guided RR mechanism 
}\label{alg:XRand}
\KwIn{model $f$, dataset $\calD$,
aggregated explanation $w$,
$\varepsilon$, $k$, $\tau$, test sample $x$ }
\KwOut{$\mathbb{S}$, $\varepsilon$-LDP $w'_x$}
\begin{algorithmic}[1]
\STATE \textit{\textbf{Step 1 - At aggregated explanation:}}
\begin{ALC@g}
\FOR{ $x_n \in \calD$}
\STATE  Compute $\calL(x_n)$ \# using Eq.~\ref{loss 2}
\FOR{$i \in [1,k]$, $j \in [k+1, k+\tau]$}
\STATE Compute $\calL(x_n)(i,j)$ \# using Eq.~\ref{loss 2}
\STATE Compute $ \Delta_{\calL}(i,j)$ \# using Eq.~\ref{delta L}
\ENDFOR
\ENDFOR
\STATE \textbf{Randomizing $w$}: \\$w' \gets \textsc{XRand} (w, \varepsilon, k, \tau, \Delta_{\calL}(i,j))$ \# using Eq. \ref{globalRR}
\STATE \textbf{Return} $\mathbb{S}$
\end{ALC@g}
\STATE \textit{\textbf{Step 2 - At sample-level explanation:}}
\begin{ALC@g}
\STATE $w_x \gets$ SHAP explanation for $x$
\STATE $w'_x \gets$ Solve the optimization problem in (\ref{optimization the loss L2}) 
\STATE \textbf{Return} $w'_x$
\end{ALC@g}
\end{algorithmic} 
\end{algorithm}

\paragraph{Algorithm Overview.} To answer this question, we first integrate the explanation loss caused by potential changes of features in the aggregated explanation into the randomized probabilities to adaptively randomize each feature in the top-$k$. Then, we minimize the explanation loss on each sample while ensuring the order of the features at the aggregated explanation follows the results of the first step. By doing so, we are able to optimize the trade-off between the model explainability and the privacy budget $\varepsilon$ used in \textsc{XRand}, as verified both theoretically (Section \ref{ssec:tradeoff}) and experimentally (Section \ref{sec:exp}). The pseudo-code of \textsc{XRand} is shown in Alg. \ref{alg:XRand}.

  





\subsection{LDP-preserving Explanations}

\paragraph{Step 1 (Alg.~\ref{alg:XRand}, lines 1-10).} We first compute the aggregated explanation $w$ over the samples of the proprietary dataset $w = \sum_{x\in \calD} w_x$. Then we sort $w$ in descending order and retain a mapping $v: \mathbb{N} \rightarrow \mathbb{N}$ from the sorted indices to the original indices. Given that $\tau$ is a predefined threshold to control the range of out-of-top-$k$ features that some of top-$k$ features can swap with, and $\beta$ is a parameter bounded in Theorem \ref{theorem-beta-bound} under a privacy budget $\varepsilon$, \textsc{XRand} defines the probability of flipping a top-$k$ feature $i$ to an out-of-top-$k$ feature $j$ as follows:
\begin{align}
\nonumber &\forall i \in [1, k], j \in [k+1,k+\tau], \tau \ge k:\\
& i = \left \{
  \begin{aligned}
    & i, \text{ with probability } p_i = \frac{\exp({\beta})}{\exp({\beta}) + \tau -1}, \\
    & j, \text{ with probability } q_{i,j} =  \frac{\tau -1}{\exp({\beta})  + \tau -1} q_j
  \end{aligned} \right. 
  \label{globalRR}
\end{align}
where $q_{j} = \frac{\exp(-\Delta_{\calL} (i,j))}{ \sum_{t \in [k+1,k+\tau]} \exp(-\Delta_{\calL}(i,t))}$ and $\Delta_{\calL}(i,j)$ is the aggregated changes of $\calL$ (Eq.~\ref{loss 2}) when flipping features $i$ and $j$, which is calculated as follows:
\begin{align}
    \Delta_{\calL}(i,j) = \frac{1}{N} \sum_{n=1}^N (| \calL(x_n) - \calL(x_n)(i,j) |)
    \label{delta L}
\end{align}
where $\calL(x_n)$ is the original loss $\calL$ of a sample $x_n\in\calD$ and $\calL(x_n)(i,j)$ is the loss  $\calL$ of the sample $x_n$ after flipping features $i$ and $j$ (Alg. \ref{alg:XRand}, lines 3,5).





After randomizing the aggregated explanation, we obtain the set $\mathbb{S}$ of features that need to be flipped in the aggregated explanation, as follows:
\begin{equation}
 \mathbb{S} = \{(i,j) | i \text{ and } j \text{ are flipped, } i \in [1,k], j \in [k+1, k+\tau] \}
\end{equation}

\paragraph{Step 2 (Alg.~\ref{alg:XRand}, lines 11-14).} For each input test sample $x$, we proceed with sample-level explanation for finding the noisy explanation $w_x'$. First, we generate a set of constraints $\mathbb{Q} = \{(i,j) | w_{x_i}' \leq w_{x_j}'\}$ that is \textit{sufficient} for $\mathbb{S}$. In particular, for each pair $(i,j)\in \mathbb{S}$, we add the following pairs to $\mathbb{Q}$:
$$(v(i+1), v(j)); (v(j), v(i - 1)); (v(i), v(j -1)); (v(j+1), v(i))$$
Given $w_x$ as the SHAP explanation of $x$, we aim to find $\phi \in \mathbb{R}^d$ such that $w_x' = w_x + \phi$ satisfies the constraints in $\mathbb{Q}$ while minimizing the loss $\calL$. To obtain $\phi$, we solve the following optimization problem:
{\small
\begin{align}  
\label{optimization the loss L2}
 \min_{\phi} &\sum_{z\in N(x)} \left((w_x + \phi)^T z - f(z)\right)^2 \exp \left(-\frac{\|z-x\|^2}{\sigma^2}\right) + \lambda \|\phi\| \\
\nonumber \mbox{s.t.\, }
& w_{x_i} + \phi_i \leq w_{x_j} + \phi_j, \quad \forall (i,j)\in \mathbb{Q} \\
\nonumber & \phi_i = 0 \quad \quad \qquad \qquad \qquad \forall i \notin \mathbb{Q}
\end{align}
}

\noindent where $\lambda$ is a regularization constant. 

The resulting noisy explanation will be $w_x' = w_x + \phi$. This problem is convex and can be solved by  convex optimization solvers 
\citep{kingma2014adam,diamond2016cvxpy}. 

\subsection{Privacy Guarantees of \textsc{XRand}}



To bound privacy loss of \textsc{XRand}, we need to bound $\beta$ in Eq. \ref{globalRR} such that the top-$k$ features in the explanation $w'$ preserves LDP, as follows:

\begin{theorem}\label{theorem-beta-bound}
Given two distinct explanations $w$ and $\widetilde{w}$ and a privacy budget $\epsilon_i$, \textsc{XRand} satisfies $\varepsilon_i$-LDP in randomizing each feature $i$ in top-$k$ features of $w$, i.e., $\frac{P(\textsc{XRand}(w_i) = z |w)}{P(\textsc{XRand}(\widetilde{w}_i) = z |  \widetilde{w})} \leq \exp(\varepsilon_i)$, if:  
\begin{equation}
\beta \le  \varepsilon_i + \ln(\tau-1) + \ln (\min  \frac{\exp(-\Delta_{\calL} (i,j))}{ \sum_{t=k+1}^{k+\tau} \exp(-\Delta_{\calL}(i,t))} ) \nonumber
\end{equation}
where $z \in Range(\textsc{XRand})$. \textbf{Proof:} See Appx. \ref{app:proof1}.
\end{theorem} 


Based on Theorem \ref{theorem-beta-bound}, the total privacy budget $\varepsilon$ to randomize all 
top-$k$ features 
is the sum of all the privacy budget $\epsilon_i$, i.e., $ \varepsilon = \sum_{i=1}^k \varepsilon_i$, since each feature $i$ is randomized independently. From Theorem \ref{theorem-beta-bound} and Eq. \ref{globalRR}, it can be seen that as the privacy budget $\varepsilon$ increases, $\beta$ can increase and the flipping probability $q_{i,j}$ decreases. As a result, we switch fewer features out of top-$k$.

\paragraph{Privacy and Explainability Trade-off.}\label{ssec:tradeoff}
To understand 
the privacy and explainability trade-off, we analyze the data utility of \textsc{XRand} mechanism through the sum square error (SSE) of the original explanation $w$ and the one   resulting from \textsc{XRand} $w'$. The smaller the SSE is, the better data utility the randomization mechanism achieves. 

\begin{theorem}\label{utility}
Utility of \textsc{XRand}: $SSE = \sum_{x \in \mathcal{D}} \sum_{i=1}^{d} (w'_{x_i} - w_{x_i})^2 = \sum_{x \in \mathcal{D}} \sum_{i=1}^{k+\tau}  (w'_{x_i} - w_{x_i})^2$, where $d$ is the number of features in the explanation.
\end{theorem}  

\begin{proof}
It is easy to see that
we only consider the probability of flipping the top-$k$ features to be out-of-the-top-$k$ up to the feature $k+\tau$. Thus, all features after $k+\tau$, i.e., from $k+\tau+1$ to $d$ are not changed. Hence the theorem follows.
\end{proof}


From the theorem, at the same $\varepsilon$, the smaller the $\tau$, the higher the data utility that \textsc{XRand} achieves. Intuitively, if $\tau$ is large, it is more flexible for the top-$k$ features to be flipped out, but it will also impair the model explainability since the original top-$k$ features are more likely to be moved far away from the top $k$. With high values of $\varepsilon$, we can obtain a smaller SSE, thus, achieving better data utility. The effect of $\varepsilon$ and $\tau$ on the SSE value is illustrated in Fig. \ref{sse-eps-tau} (Appendix). 

\section{Certified Robustness}\label{sec:rob}





Our proposed \textsc{XRand} can be used as a defense against the XBA since it protects the top-$k$ important features. We further establish the connection between \textsc{XRand} with certified robustness against XBA. 
Given a data sample $x$: 1) In the training time, we guarantee that up to a portion of poisoning samples in the outsourced training data, XBA fails to change the model predictions; and 2) In the inference time, we guarantee that up to a certain backdoor trigger size, XBA fails to change the model predictions. A primer on certified robustness is given in Appx. \ref{app:cert}.



 \subsection{Training-time Certified Robustness }

We consider  the original training data $\mathcal{D}$ as the proprietary  data, and  the explanation-guided backdoor samples $\mathcal{D}_o$ as the outsourced data inserted into the proprietary  data. 
The outsourced data $\mathcal{D}_o$ alone may not be sufficient to train a good classifier. In addition, the outsourced data inserted into propriety data can lessen the certified robustness bound of the propriety data. Therefore, we cannot quantify the certified poisoning training size of the outsourced data $\mathcal{D}_o$ directly by applying a bagging technique \citep{jia2020intrinsic}. To address this problem, we 
quantify the certified poisoning training size $r$ of $\mathcal{D}_o$ against XBA by uncovering its correlation with the poisoned training data $\mathcal{D}' = \mathcal{D} \cup \mathcal{D}_o$.

Given a model prediction on a data sample $x$ using $\mathcal{D}$, denoted as $f(\mathcal{D}, x)$, we ask a simple question: ``What is the minimum number poisoning data samples, i.e., certified poisoning training size $r_\mathcal{D}$, added into $\mathcal{D}$ to change the model prediction on $x$: $f(\mathcal{D}, x) \neq f(\mathcal{D}_+, x)$?" After adding $\mathcal{D}_o$ into $\mathcal{D}$, we ask the same question: ``What is the minimum number poisoning data samples, i.e., certified poisoning training size $r_{\mathcal{D}'}$, added into $\mathcal{D}' = \mathcal{D} \cup \mathcal{D}_o$ to change the model prediction on $x$: $f(\mathcal{D}', x) \neq f(\mathcal{D}'_{+}, x)$?" The difference between $r_\mathcal{D}$ and $r_{\mathcal{D}'}$ provides us a certified poisoning training size on $\mathcal{D}_o$. Intuitively, if $\mathcal{D}_o$ does not consist of poisoning data examples, then $r_\mathcal{D}$ is expected to be relatively the same with $r_{\mathcal{D}'}$. Otherwise, $r_{\mathcal{D}'}$ can be significantly smaller than $r_{\mathcal{D}}$ indicating that $\mathcal{D}_o$ is heavily poisoned with at least $r =  r_{\mathcal{D}} - r_{\mathcal{D}'}$ number of poisoning data samples towards opening backdoors on $x$.

\begin{theorem}
Given two certified poisoning training sizes $r^*_{\mathcal{D}} = \arg\min_{r_\mathcal{D}} r_\mathcal{D}$ and $r^*_{\mathcal{D}'} = \arg\min_{r_{\mathcal{D}'}} r_{\mathcal{D}'}$, the certified poisoning training size $r$ of the outsourced data $\mathcal{D}_o$ is:
\begin{equation} 
    r = r_{\mathcal{D}}^* - r_{\mathcal{D}'}^*
    \label{certified robustness bound}
\end{equation}
\textbf{Proof:} Refer to Appx. \ref{Proof of theorem certified training} for the proof and its tightness.
\label{certified poisoning training theory}
\end{theorem}

\subsection{Inference-time Certified Robustness}

It is not straightforward to adapt existing certified robustness bounds at the inference-time into \textsc{XRand}, since there is a gap between model training as in existing approaches \cite{jia2020intrinsic,lecuyer2019certified,phan2020scalable} and the model training with explanation-guided poisoned data as in our system. Existing approaches can randomize the data input $x$ and then derive certified robustness bounds given the varying output. This typical process does not consider the explanation-guided poisoned data that can potentially affect certified robustness bounds in our system. 
To address this gap, note
that we can always find a mechanism to inject random noise into the data samples $x$ such that the samples achieve the same level of DP guarantee as the explanations. Based on this, we can generalize existing certified bounds against XBA at the inference time in \textsc{XRand}.

When explanation-guided backdoor samples are inserted into the training data, upon bounding the sensitivity that the backdoor samples change the output of $f$, there always exists a noise $\alpha$
that can be injected into a benign sample $x$, i.e., $x+\alpha$, to achieve an equivalent $\varepsilon$-LDP protection. Given the explanation $w_x$ of $x$, we focus on achieving a robustness condition to $L_p(\mu)$-norm attacks, where $\mu$ is the radius of the norm ball, formulated as follows:
\begin{equation}
    \forall \alpha \in L_p(\mu): f_l({x} + \alpha|w_x) > f_{\neg l} ({x} + \alpha|w_x)
    \label{robustness condition}
\end{equation}
where $l\in \{0, 1\}$ is the true label of $x$ and $\neg l $ is the NOT operation of $l$ in a binary classification problem.





There should exist a correlation among $\alpha$ and $w_x$ that needs to be uncovered in order to bound the robustness condition in Eq.~\ref{robustness condition}. However, it is challenging to find a direct mapping function $\mathcal{M}: w_x \rightarrow \alpha$ so that when we randomize $w_x$, the change of $\alpha$ is quantified. 
We address this challenge by quantifying the sensitivity of $\alpha$ given the average change of the explanation of multiple samples $x \in \mathcal{X}$, as follows:
\begin{equation}
    \Delta_{\alpha | w} = \frac{1}{|\mathcal{X}|d}  \sum_{x \in \mathcal{X}}|w_x - w'_x|_1 
    \label{alpha w}
\end{equation}
where $|\mathcal{X}|$ is the size of $\mathcal{X}$.


$\Delta_{\alpha | w}$ can be considered as a bounded sensitivity of \textsc{XRand} given the input $x$ since: \textbf{(1)} We can achieve the same DP guarantee by injecting Laplace or Gaussian noise into the input $x$ using the sensitivity $\Delta_{\alpha | w}$; and \textbf{(2)} The explanation perturbation happens only once and is permanent, that is, there is no other bounded sensitivity associated with the one-time explanation perturbation. The sensitivity $\Delta_{\alpha | w}$ establishes a new connection between explanation perturbation and the model sensitivity given the input sample $x$. That enables us to derive robustness bounds using different techniques, i.e., PixelDP \cite{lecuyer2019certified,Phan2019HeterogeneousGM} and (boosting) randomized smoothing (RS) \cite{horvath2022boosting, cohen2019certified}, since we consider the sensitivity $\Delta_{\alpha | w}$ as a part of randomized smoothing to derive and enhance certified robustness bounds. 

The rest of this section only discusses the bound using PixelDP, we refer the readers to Appx. \ref{app:pixeldp} for the bound using boosing RS. 
Given a randomized prediction $f(x)$ satisfying $(\varepsilon, \delta)$-PixelDP w.r.t. a $L_p(\mu)$-norm metric, we have:
\begin{equation}
    \forall l \in \{0, 1\}, \forall \alpha \in L_p(\mu): \mathbb{E} f_l(x) \le e^{\varepsilon} \mathbb{E} f_l (x+ \alpha) + \delta
    \label{certified bound}
\end{equation}
where $\mathbb{E} f_l(x)$ is the expected value of $f_l(x)$, $\varepsilon$ is a predefined privacy budget, and $\delta$ is a broken probability. When we use a Laplace noise, $\delta = 0$.

We then apply PixelDP with the  sensitivity $\Delta_{\alpha | w}$ and  a noise standard deviation $\sigma = \frac{\Delta_{\alpha | w} \mu}{\varepsilon}$ for Laplace noise, or $\sigma = \frac{\Delta_{\alpha | w} \mu \sqrt{2\ln{(1.25/\delta)}}}{\varepsilon}$ for Gaussian noise. From that, when maximizing the attack trigger's magnitude $\mu$:  $\mu_{\max} = \max_{\mu \in \mathbb{R}^+} \mu$ such that
the  generalized robustness condition (Eq. \ref{certified bound}) holds, the prediction on $x$ using \textsc{XRand} is robust against the XBA up to $\mu_{\max}$. As a result,  we have a robustness  certificate of $\mu_{\max}$ for $x$.

\section{Experiments} \label{sec:exp}


To evaluate the performance of our defense, we conduct the XBA proposed by \cite{severi2021explanation} against the explanations returned by \textsc{XRand} to create backdoors on cloud-hosted malware classifiers. Our experiments aim to shed light on understanding (1) the effectiveness of the defense in mitigating the XBA, and
(2) the faithfulness of the explanations returned by \textsc{XRand}. 
The experiments are conducted using LightGBM \citep{2018arXiv180404637A} and EmberNN \citep{severi2021explanation} classification models that are trained on the EMBER \citep{2018arXiv180404637A} dataset. A detailed description of the experimental settings can be found in Appx. \ref{app:exp}. We quantify \textsc{XRand} via the following metrics:

{\it \underline {Attack success rate.}} This metric is defined as the portion of trigger-embedded malware samples that are classified as goodware by the backdoored model. Note that we only embed the trigger into malware samples that were classified correctly by the clean model. The primary goal of our defense is to reduce this value.

{\it \underline {Log-odds.}} To evaluate the faithfulness of our \textsc{XRand} explanation, we compare the log-odds score of the \textsc{XRand} explanations with that of the ones originally returned by SHAP. Based on the explanation of a sample, the log-odds score is  computed by identifying the top 20\% important features that, if erased, can change the predicted label of the sample \citep{shrikumar2017learning}. Then we obtain the change in log-odd of the prediction score of the original sample and the sample with those features erased. The higher the log-odds score, the better the explanation in terms of identifying important features. To maintain a faithful explainability, the \textsc{XRand} explanations should have comparable log-odds scores as the original SHAP explanations.

\begin{figure}[t]
    \centering
	\subfloat[LightGBM]{
	    \includegraphics[width=0.45\linewidth]{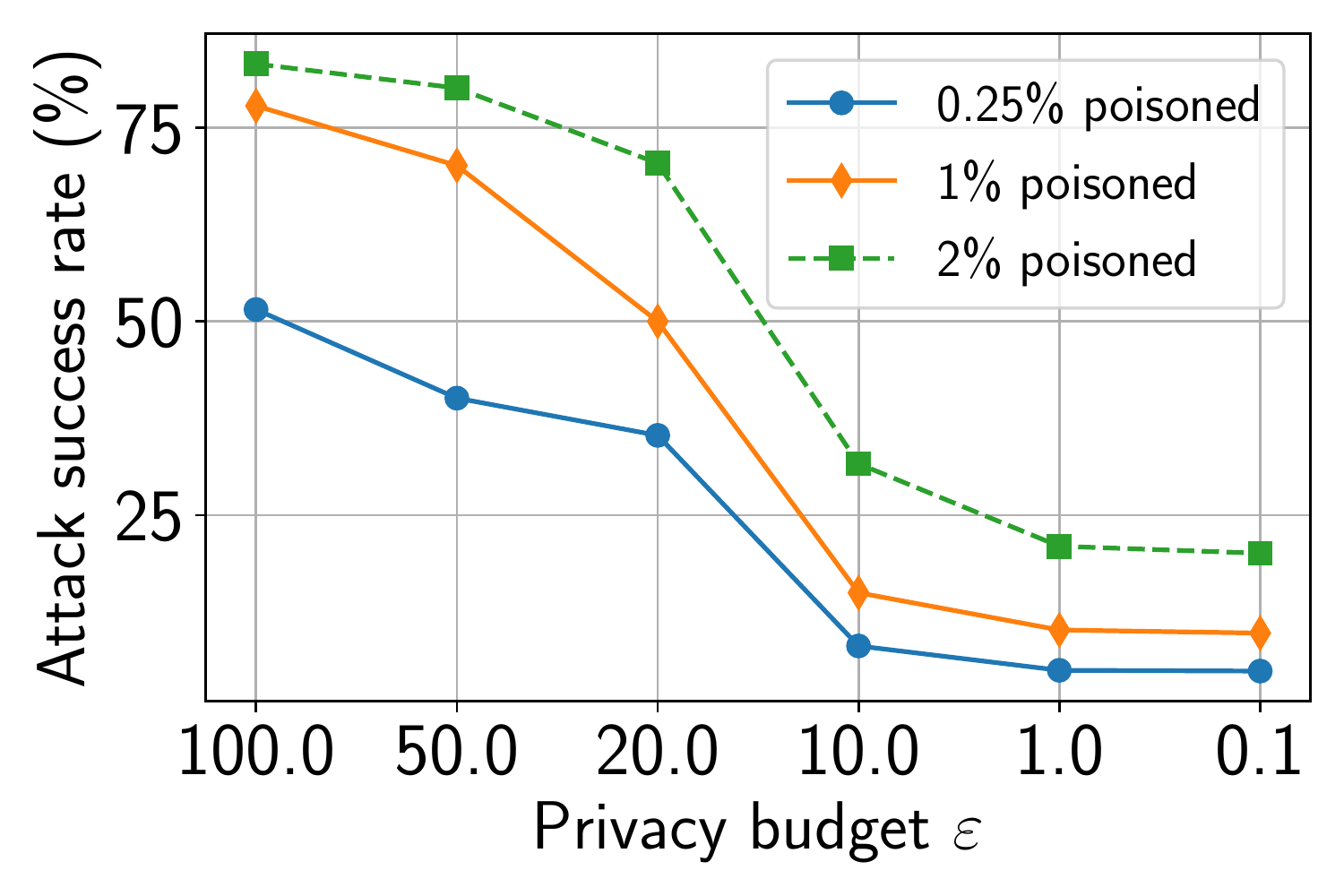}
    	\label{fig:attpoison-gbm}
	}%
 	\hfill
	\subfloat[EmberNN]{
	    \includegraphics[width=0.45\linewidth]{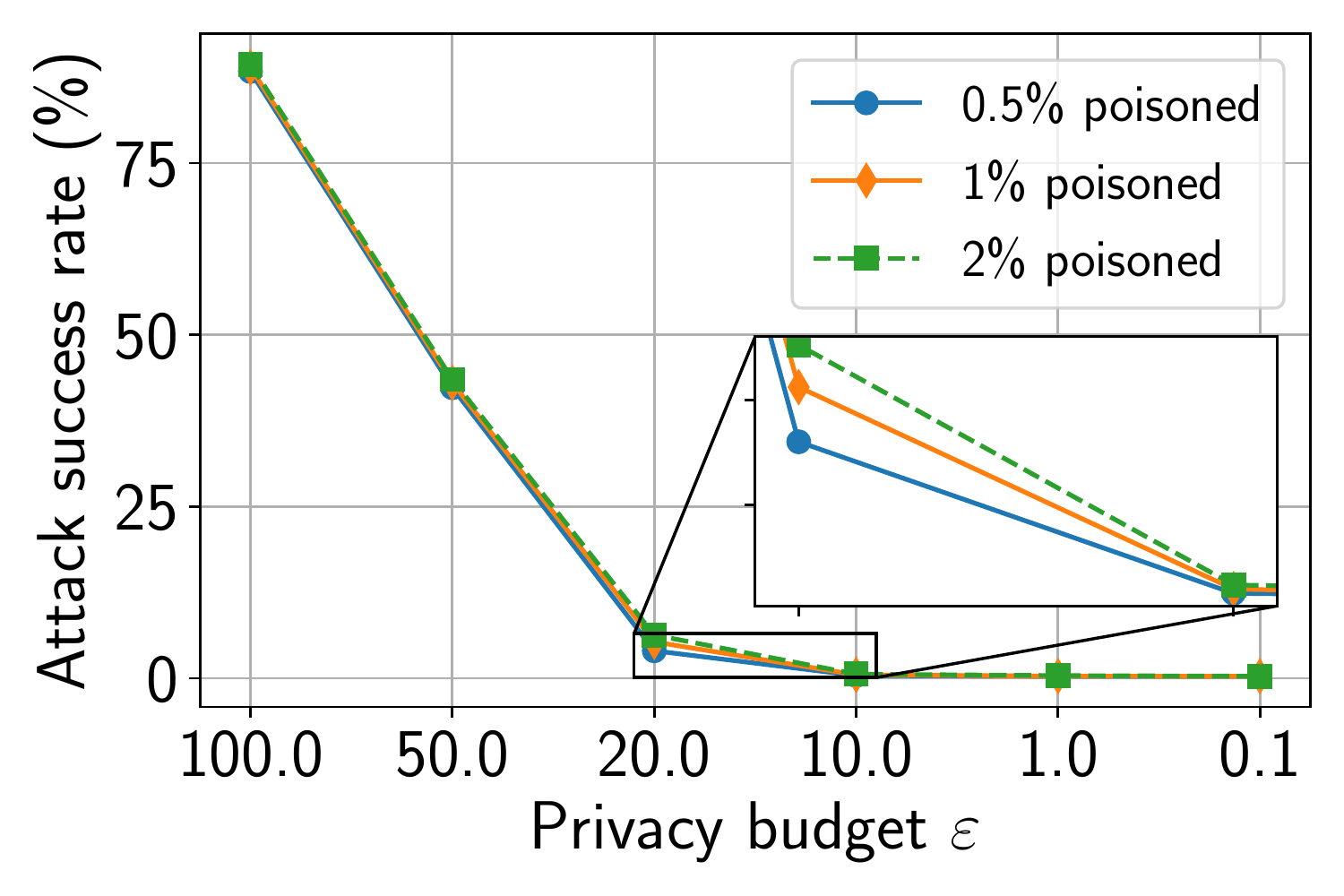}
	    \label{fig:attpoison-ember}
    }
    \caption{Attack success rate as a function of privacy budget $\varepsilon$ and the portion of poisoned samples on LightGBM and EmberNN. The trigger size is fixed at 10 and 16 for LightGBM and EmberNN, respectively.}
    \label{fig:attpoison}
\end{figure}



\paragraph{Attack Mitigation.} 
We observe the attack success rate of XBA when our \textsc{XRand} explanations are used to craft backdoor triggers. We set $k$ to be equal to the trigger size of the attack, and fix the predefined threshold $\tau=50$. Fig. \ref{fig:attpoison} highlights the correlation between the attack success rate and the privacy budget $\varepsilon$ in our defense. Intuitively, the lower the $\varepsilon$, the more obfuscating the top goodware-oriented features become. Hence, Figs. \ref{fig:attpoison-gbm} and \ref{fig:attpoison-ember} show that the attack success rate is greatly diminished as we tighten the privacy budget, since the attacker has less access to the desired features. Moreover, in a typical backdoor attack, injecting more poisoned samples into the training data makes the attack more effective. Such behavior is exhibited in both LightGBM (Fig. \ref{fig:attpoison-gbm}) and EmberNN (Fig. \ref{fig:attpoison-ember}), though EmberNN is less susceptible to the increase of poisoned data. However, in practice, the attacker wishes to keep the number of poisoning samples relatively low to remain stealthy. At a 1\% poison rate, our defense manages to reduce the attack success rate from 77.8\% to 10.2\% with $\varepsilon = 1.0$ for LightGBM. It performs better with EmberNN where the attack success rate is reduced to 5.3\% at $\varepsilon = 10.0$.

\begin{figure}[t]
    \centering
	\subfloat[LightGBM]{
	    \includegraphics[width=0.45\linewidth]{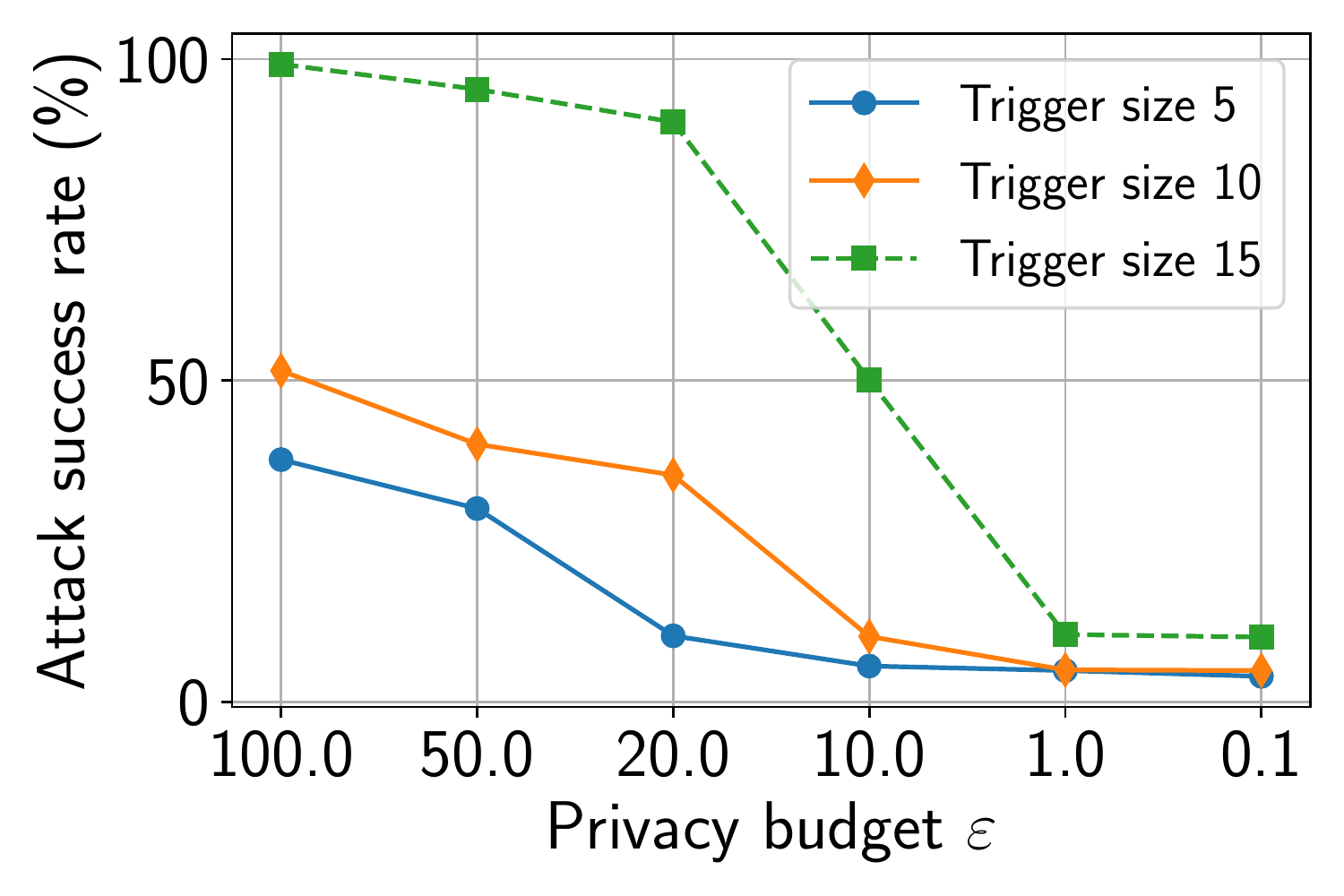}
    	\label{fig:atttrigger-gbm}
	}%
	\hfill
	\subfloat[EmberNN]{
	    \includegraphics[width=0.45\linewidth]{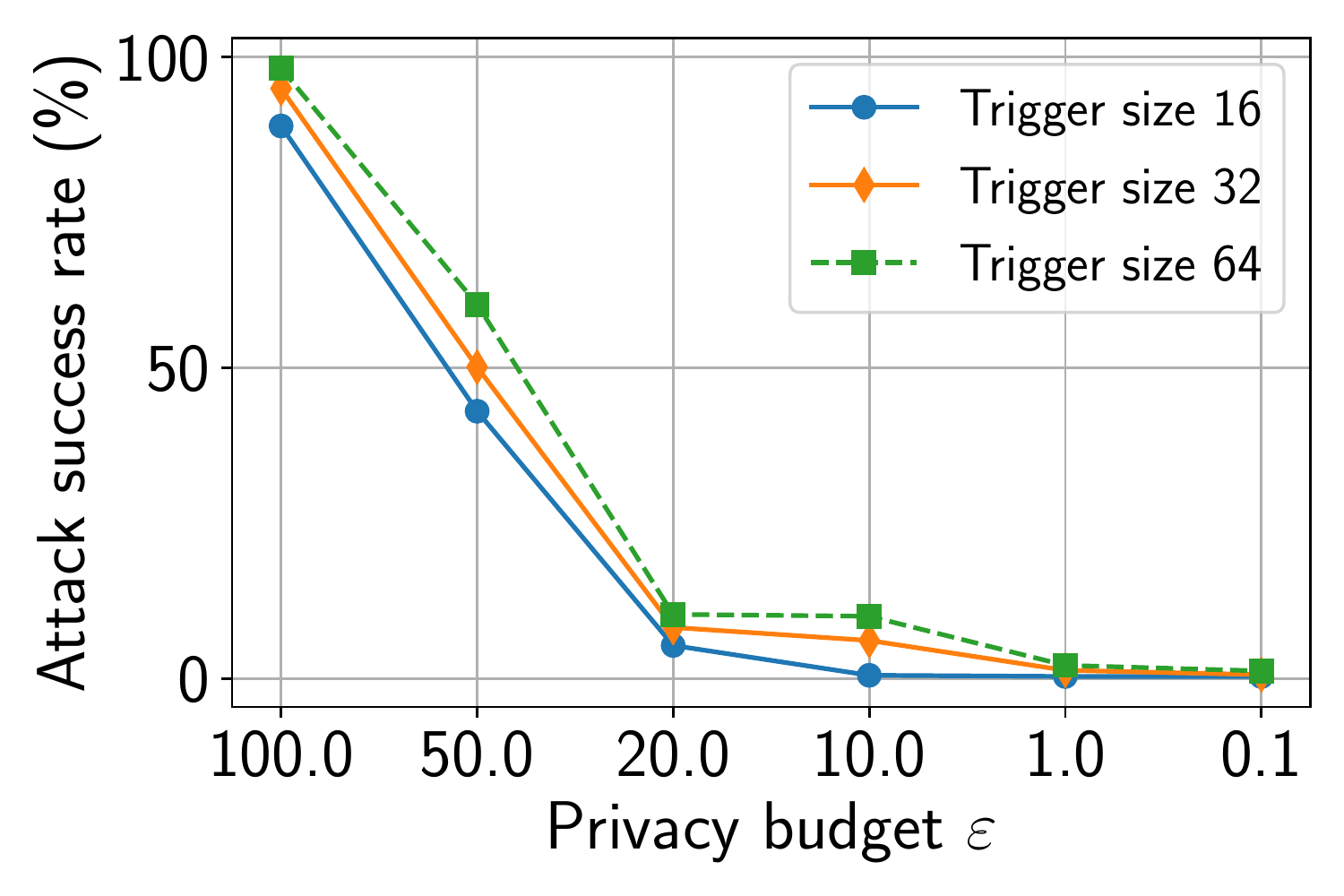}
	    \label{fig:atttrigger-ember}
    }
    \caption{Attack success rate as a function of trigger size and privacy budget $\varepsilon$. We vary the trigger size and fix the portion of poisoned samples at 0.25\% and 1\% for LightGBM and EmberNN, respectively. 
    }
    \label{fig:atttrigger}
\end{figure}



Additionally, we examine the effect of the trigger sizes on our defense. The trigger size denotes the number of features that the attacker modifies to craft poisoned samples. We vary the trigger size consistently with previous work \citep{severi2021explanation}. In backdoor attacks, a larger trigger makes the backdoored model more prone to misclassification, thus improving the attack success rate. Fig. \ref{fig:atttrigger} shows that the attack works better with large triggers on both models, however, as aforementioned, the attacker would prefer small trigger sizes for the sake of stealthiness. This experiment shows that, for the trigger sizes that we tested, our proposed defense can successfully maintain a low attack success rate given a wide range of the privacy budget $\varepsilon \in [0.1, 10]$ (Fig. \ref{fig:attpoison-gbm}, \ref{fig:attpoison-ember}).



\begin{figure}[t]
    \centering
    \includegraphics[width=0.55\linewidth]{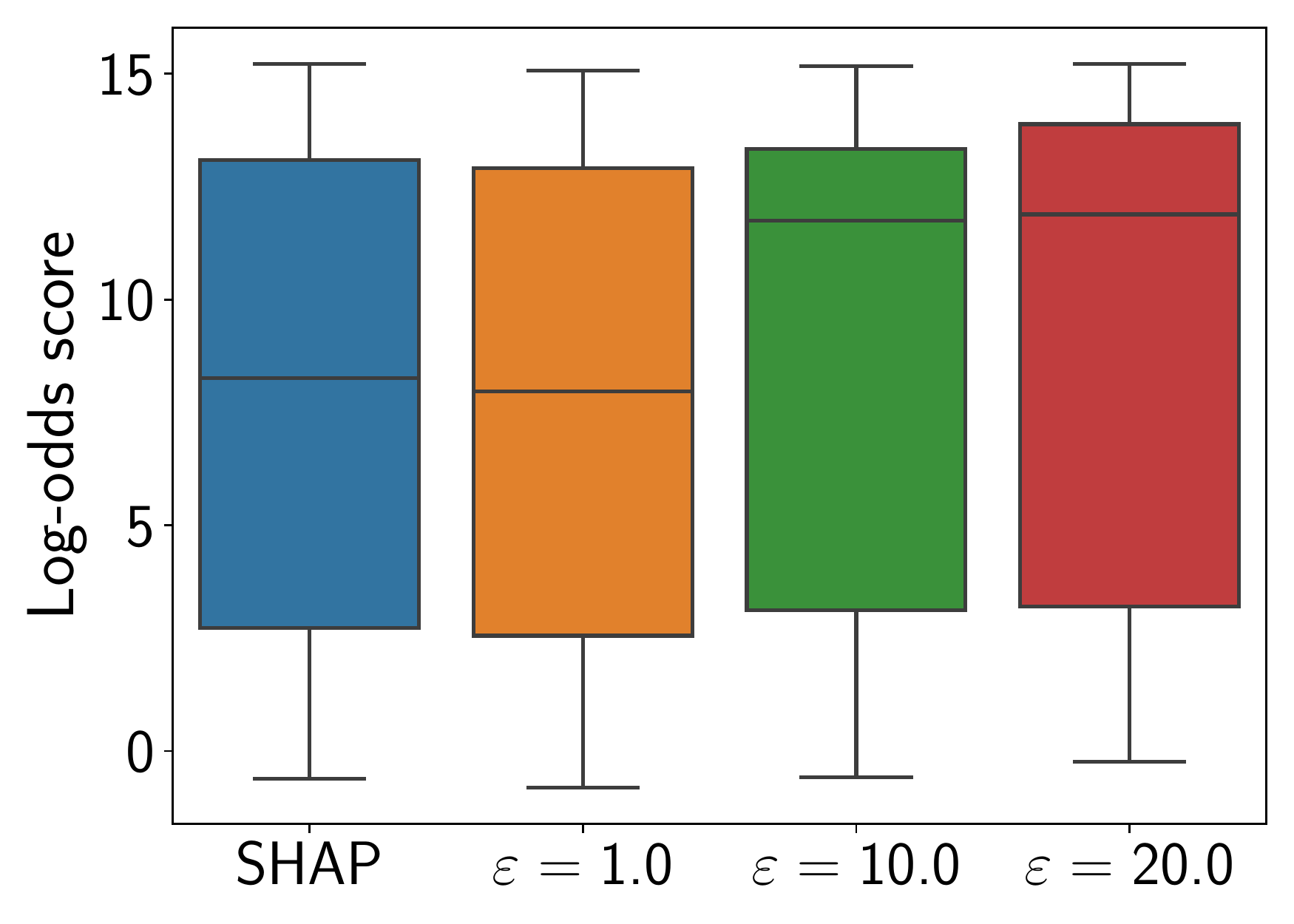}
    \caption{Log-odds score of the explanations of 20,000 goodware and malware samples. The leftmost box represents the original explanations by SHAP, the remaining boxes illustrate the explanations after applying \textsc{XRand} at $\varepsilon = 1.0, 10.0, 20.0$, respectively.}
    \label{fig:logodd}
\end{figure}

We refer the readers to Appx. \ref{app:exp} for more experimental results on additional malware datasets and the evaluation of our robustness bounds. In short, regarding the training-time bound, we observe that a smaller privacy budget $\varepsilon \in [0.1, 10]$ results in a more robust model against the XBA. As for the inference-time bound, we obtain high certified accuracy (Eq. \ref{certified accuracy}) under rigorous privacy budgets, i.e.,  $89.17\%$ and $90.42\%$ at $\varepsilon \in [0.1, 1.0]$. We notice that the PixelDP-based bound attains a stronger performance than using boosting RS. This is because boosting RS is not designed for models trained on backdoored data like \textsc{XRand}.

\paragraph{Faithful Explainability.} From the previous results, we observe that a wide range of the privacy budget $\varepsilon \in [0.1, 10]$ provides a good defense against the XBA. The question remains whether the explanations resulting from these values of $\varepsilon$ are still faithful. Fig. \ref{fig:logodd} shows the log-odds score of the original explanations returned by SHAP and of the ones after applying our \textsc{XRand} mechanism. The \textsc{XRand} explanations at $\varepsilon = 1.0, 10.0$ have comparable log-odds scores to those of SHAP. This is because our defense works with small values of $k$ (e.g., $k=10$). Therefore, \textsc{XRand} only randomizes the SHAP values within a small set of top goodware-oriented features. As a result,  \textsc{XRand} can still capture the important features in its explanations.

\begin{figure}
    \centering
	\subfloat[SHAP explanation (without \textsc{XRand})]{
	    \includegraphics[width=1\linewidth]{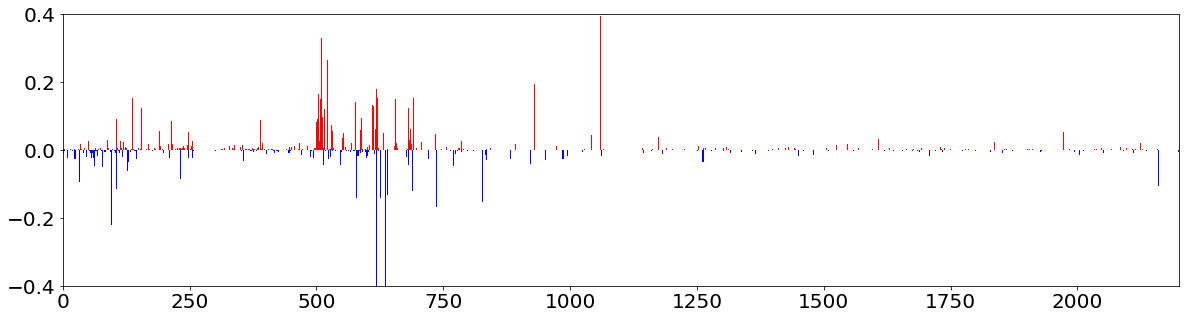}
    	\label{fig:expori}
	}%
	\hfill
	\subfloat[\textsc{XRand} explanation]{
	    \includegraphics[width=1\linewidth]{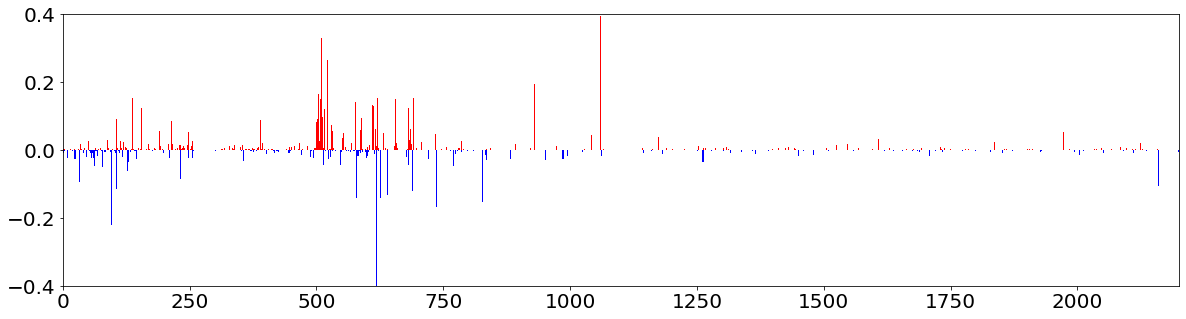}
	    \label{fig:expnew}
    }
    \caption{Visualizing the SHAP explanation and our \textsc{XRand} explanation of a test sample. The plot shows the SHAP value of each feature. The red vertical lines represent positive SHAP values that indicate malware-oriented features, while the blue vertical lines are negative SHAP values indicating goodware-oriented features.}
    \label{fig:expcmp}
\end{figure}

Furthermore, we visualize the explanation of a test sample before and after applying  \textsc{XRand} in Figs. \ref{fig:expori} and \ref{fig:expnew}, respectively. As can be seen, the SHAP values of the two explanations largely resemble one another, except for minor differences in less than 10 features (out of 2,351 features). Importantly, the \textsc{XRand} explanation evidently manifests similar sets of important malware and goodware features as the original explanation by SHAP, which also explains the comparable log-odds score in Fig. \ref{fig:logodd}. More visualizations on \textsc{XRand} can be found in Appx. \ref{app:viz}.


\section{Related Work} \label{sec:rel}
\paragraph{Vulnerability of Explanations and Backdoor Defenses.} 
There are two main lines of research in studying the vulnerability of model explanations. First, model explainers are prone to adversarial manipulations, thereby making the explanations unreliable and untrustworthy \cite{heo2019fooling, slack2020fooling, zhang2020interpretable, ghorbani2019interpretation}. Second, the information leaked through the explanations can be exploited to conduct several attacks against the black-box models \cite{shokri2021privacy,milli2019model,miura2021megex,zhao2021exploiting,severi2021explanation}. 

Our work focuses on the second direction in which we show that MLaaS is particularly susceptible to this type of attack. Although such vulnerability has been studied before, the discussion on its defense mechanism is still limited. To that extent, we propose a new concept of LDP explanation such that we can obfuscate the explanations in a way that limits the knowledge that can be leveraged by the attackers while maintaining the faithfulness of the explanations. 

Moreover, several defenses have been proposed to tackle the backdoor attacks \cite{liu2018fine, tran2018spectral, wang2019neural}. 
As evidenced by \citep{severi2021explanation} that the backdoors created by the XBA remain stealthy against these defenses. Thus,
our \textsc{XRand} is the first solution that can tackle XBA in MLaaS.

\paragraph{LDP Mechanisms.} 
Existing LDP mechanisms can be classified into four lines.  \textbf{1)} Direct encoding approaches \citep{warner1965randomized,kairouz2016discrete} are to directly apply RR mechanisms to a binary response to output the true and false responses with different probabilities. 
\textbf{2)} Generalized RR approaches  \citep{duchi2019lower,wang2019collecting,zhao2020local} directly perturb the numerical (real) value of data into a finite set or continuous range of output values.  \textbf{3)} Hash-based approaches \citep{erlingsson2014rappor,wang2017locally,bassily2015local,acharya2019hadamard,wang2019murs}. 
\textbf{4)} Binary encoding-based approaches \citep{arachchige2019local,lyu2020towards} converts input $x$ into a binary encoding vector of $0$ and $1$, then independently applies RR mechanisms on these bits. Despite the effectiveness of these RR mechanisms, existing approaches have not been designed for XBA, in which the explanations expose the training data background.

\section{Conclusion} \label{sec:con}
In this paper, we have shown that, although explanations help improve the understanding and interpretability of black-box models, they also leak essential information about the inner workings of the models. Therefore, the black-box models become more vulnerable to attacks, especially in the context of MLaaS where the prediction and its explanation are returned for each query. 
With a novel two-step LDP-preserving mechanism, we have proposed \textsc{XRand} to protect the model explanations from being exploited by adversaries via obfuscating the top important features, while maintaining 
the faithfulness of explanations. 

\section*{Acknowledgments}
This material is based upon work supported by the National Science Foundation under grants IIS-2123809, CNS-1935928, and CNS-1935923.

\bibliography{aaai23}




\clearpage
\appendix

\section{Proof of Theorem 1}\label{app:proof1}
\begin{proof}
Let us denote $w_i$ as the  SHAP score of feature $i$ in the aggregated explanation. Given  the two explanations $w$ and   $\widetilde{w}$ that can be different at any feature and any possible output $z \in Range(\textsc{XRand})$, where $Range(\textsc{XRand})$ denotes every possible output of \textsc{XRand},
we have:
\begin{align}\label{proof-beta}
 \nonumber & \frac{P(\textsc{XRand}(w_i) = z )}{P(\textsc{XRand}(\widetilde{w}_i) = z )} \le \frac{\max P(\textsc{XRand}(w_i) = z )}{\min P(\textsc{XRand}(\widetilde{w}_i) = z )} \\
  \nonumber & = \frac{ \frac{\exp({\beta})}{\exp({\beta}) + \tau -1} }{  \min (\frac{\exp(-\Delta_{\calL} (i,j))}{ \sum_{t \in [k+1,k+\tau]} \exp(-\Delta_{\calL}(i,t))} \frac{\tau - 1}{\exp({\beta}) + \tau -1})  } \\
  &= \frac{ \exp(\beta ) }{(\tau -1) \min ( \frac{\exp(-\Delta_{\calL} (i,j))}{ \sum_{t \in [k+1,k+\tau]} \exp(-\Delta_{\calL}(i,t))} ) } \le \exp(\varepsilon_i)
\end{align}

Taking a natural logarithm of Eq. \ref{proof-beta}, we obtain:
\begin{align}  
  \nonumber &\nonumber \ln ( \frac{ \exp(\beta ) }{(\tau -1) \min ( \frac{\exp(-\Delta_{\calL_2} (i,j))}{ \sum_{t=k+1}^{k+\tau} \exp(-\Delta_{\calL_2}(i,t))} ) }  ) \le \ln (\exp(\varepsilon_i)) \\
 \nonumber  &\Leftrightarrow \beta \le \varepsilon_i + \ln(\tau-1) + \ln (\min  \frac{\exp(-\Delta_{\calL} (i,j))}{ \sum_{t=k+1}^{k+\tau} \exp(-\Delta_{\calL}(i,t))} )
\end{align}
Consequently, Theorem \ref{theorem-beta-bound} holds. 
\end{proof}

\section{A Primer on Certified Robustness}\label{app:cert}
 The ultimate goal of certified robustness is to guarantee consistency on the model performance under data perturbation. In specific, it has to ensure that a small perturbation in the input does not change the predicted label.
Given a benign example $x$, the robustness condition to $l_p(\mu)$-norm attacks can be stated as follows:
\begin{equation}
\forall \alpha \in l_p(\mu): f_i(x + \alpha) > \max_{j: j\neq i} f_{j}(x + \alpha) 
\label{RobustCond1}
\end{equation}
where $i$ is the ground-truth label of the sample $x$. The condition essentially indicates that a small perturbation $\alpha$ in the input does not change the predicted label $i$.

\paragraph{PixelDP \cite{lecuyer2019certified}.} To achieve the robustness condition in Eq. \ref{RobustCond1}, Lecuyer et al. \cite{lecuyer2019certified} introduce an algorithm, called \textbf{PixelDP}. By considering an input $x$ 
as a database in DP parlance, and individual features 
as tuples, PixelDP shows that randomizing the function $f(x)$ to enforce DP on a small number of features in the input sample  guarantees the robustness of predictions. To randomize $f(x)$, \textit{random noise} $\sigma_r$ is injected into either input $x$ or an arbitrary hidden layer, resulting in the following $(\epsilon_r, \delta_r)$-PixelDP condition: 
\begin{lemma} $(\epsilon_r, \delta_r)$-PixelDP \citep{lecuyer2019certified}. Given a randomized scoring function $f(x)$ satisfying $(\epsilon_r, \delta_r)$-PixelDP w.r.t. a $l_p$-norm metric, we have:
\begin{equation}
\forall i, \forall \alpha \in l_p(1): \mathbb{E} f_i(x) \leq e^{\epsilon_r} \mathbb{E} f_i(x + \alpha) + \delta_r
\label{PixelDPDef}
\end{equation}
where $\mathbb{E} f_i(x)$ is the expected value of $f_i(x)$, $\epsilon_r$ is a predefined budget, $\delta_r$ is a broken probability. 
\label{LemmaPixelDP} 
\end{lemma}

At the prediction time, a certified robustness check is implemented for each prediction, as follows: 
\begin{equation}
\hat{\mathbb{E}}_{lb} f_i(x) > e^{2\epsilon_r} \max_{j: j\neq i} \hat{\mathbb{E}}_{ub} f_{j}(x) + (1 + e^{\epsilon_r})\delta_r 
\label{RobustCon2} 
\end{equation}
where $\hat{\mathbb{E}}_{lb}$ and $\hat{\mathbb{E}}_{ub}$ are the lower and upper bounds of the expected value $\hat{\mathbb{E}} f(x) = \frac{1}{n} \sum_n f(x)_n$, derived from the Monte Carlo estimation with an $\eta$-confidence, given $n$ is the number of invocations of $f(x)$ with independent draws in the noise $\sigma_r$.
Passing the check for a given input guarantees that no perturbation up to $l_p(1)$-norm (where $1$ is the radius of the $l_p$-norm ball) can change the model's prediction.

\paragraph{Boosting Randomized Smoothing \cite{horvath2022boosting}.}
Another state-of-the-art approach to certifying robustness is boosting randomized smoothing (RS).
In this scheme, an ensemble model $\bar{f}$ of $M$ classifiers trained on the same dataset with different random seeds (same structures and settings). We denote $p_1$ as the success probability that the ground-truth label $l$ is correctly predicted. Let $c_i$ and $c_{j, j \neq i}$ be the expected values of the clean model over the randomness of the training process, such as $c_i=\mathbb{E}_i(f_i(x))$ and $c_{j, j \neq i}=\mathbb{E}_{j}(f_j(x)) $. Let $t_i=\frac{c_i - c_j}{\sigma_i(M)}$ and $z_i$ be the probability distribution of the classification margin over class $i$, we have the following robustness condition 
\begin{align}
\nonumber    p_1 & := P\big(\bar{f}(x+\alpha)=i\big) \ge 1 - P\big(|z_i- (c_i -c_j)|\big) \\
    & \ge t_i \sigma_i(M) \ge 1 - \sum_{j, j\neq i} \frac{\sigma_j^2(M)}{ (c_i -c_j)^2}
    \label{p1 condition}
\end{align}
where $\sigma_i(M)$ is the variance of the classification margin $z_i$. 

It is shown in \cite{horvath2022boosting} that when $\sigma(M)$ decreases (i.e., a better certified robustness bound), the number of ensemble models $M$ increases, and the lower bound of the success probability $p_1$ approaches the ground-truth label $l$. 

Given the success probability $p_1$, confidence $\gamma$, number of samples $N$, and perturbation variance $\sigma_{\alpha}^2$ (up to an incorrect prediction), the probability distribution over the certified radius ($R$) is defined as follows:
\begin{align}
    P\Big(R = \sigma_{\alpha} \Phi^{-1} \big( \underline{p_1} (N_1,N,\alpha)  \big)  \Big) = \mathcal{B} (N_1,N,p_1) 
\end{align}
where $\Phi^{-1}$ denotes the inverse Gaussian CDF, $\mathcal{B} (N_1,N,p_1) $ is the probability of drawing $N_1$ successes in $N$ trials from a Binomial distribution with the success probability $p_1$ and $\underline{p_1}$  is the lower bound to the success probability of a Bernoulli experiment given $N_1$ success in $N$ trials with confidence $\alpha$ according to the Clopper-Pearson interval \cite{clopper1934use}. 

The certified robustness bound is $R^* = \arg\max R$ such that the condition in Eq.~\ref{p1 condition} is satisfied. 

\paragraph{Bagging ensemble learning.} For the certified robustness at the training-time against XBA, we leverage the bagging ensemble learning method-based certified robustness bounds, which have been demonstrated to be effective in defending against backdoor data poisoning attacks \citep{jia2020intrinsic}. The ensemble learning approach trains a set of base models and leverages a majority vote to quantify the difference between the lower bound of the class with the highest probability and the upper bound of the class with the second highest probability. Base upon that, one can identify the minimum number of poisoning data samples, called \textit{certified poisoning training size}, that can change the majority vote.

\section{Proof of Theorem \ref{certified poisoning training theory}}
\label{Proof of theorem certified training}

\begin{proof}

Recall that $\mathcal{D}$, $\mathcal{D}_o$, and  $\mathcal{D}'= \mathcal{D} \cup \mathcal{D}_o$ are the proprietary data, the outsourced data, and the poisoned training data, respectively. Given the model prediction on a data sample $x$ using $\mathcal{D}$, denoted as $f(\mathcal{D}, x)$, we ask a simple question: ``What is the minimum number poisoning data samples, i.e., certified poisoning training size $r_\mathcal{D}$, added into $\mathcal{D}$ to change the model prediction on $x$: $f(\mathcal{D}, x) \neq f(\mathcal{D}_+, x)$?" After adding $\mathcal{D}_o$ into $\mathcal{D}$, we ask the same question: ``What is the minimum number poisoning data samples, i.e., certified poisoning training size $r_{\mathcal{D}'}$, added into $\mathcal{D}' = \mathcal{D} \cup \mathcal{D}_o$ to change the model prediction on $x$: $f(\mathcal{D}', x) \neq f(\mathcal{D}'_{+}, x)$?" The difference between $r_\mathcal{D}$ and $r_{\mathcal{D}'}$ provides us a certified poisoning training size on $\mathcal{D}_o$. Intuitively, if $\mathcal{D}_o$ does not contain poisoning data examples, then $r_\mathcal{D}$ is expected to be relatively the same with $r_{\mathcal{D}'}$. Otherwise, $r_{\mathcal{D}'}$ can be significantly smaller than $r_{\mathcal{D}}$ indicating that $\mathcal{D}_o$ is heavily poisoned with at least $r =r_{\mathcal{D}} - r_{\mathcal{D}'} $ number of poisoning data samples towards opening backdoors on $x$. Now, our goal is to  find $r_\mathcal{D}$, $ r_{\mathcal{D}'}$, and their connection to $r$.

Let us denote two sets of poisoned training datasets with at most $r_{\mathcal{D}'}$ poisoned training samples in $\mathcal{D}'$ and at most $r_{\mathcal{D}}$ poisoned training samples in $\mathcal{D}$, namely $ B(\mathcal{D}',r_{\mathcal{D}'})$ and $ B(\mathcal{D},r_{\mathcal{D}})$, respectively,  as follows:
\begin{align} 
 & B(\mathcal{D},r_{\mathcal{D}})  = \{ \mathcal{D}_+ \text{ s.t. } |\mathcal{D}_+ | -|\mathcal{D}|  \le r_{\mathcal{D}} \} 
 \\
 & B(\mathcal{D}',r_{\mathcal{D}'})  = \{ \mathcal{D}'_+ \text{ s.t. } |\mathcal{D}'_+ | -|\mathcal{D}'|  \le r_{\mathcal{D}'} \} 
\end{align} 

We call $s(\mathcal{D})$ as a random subsample data that are sampled from $\mathcal{D}$ with replacement uniformly at random. 
We denote $p_l$ as the label probability, in which  $p_l = \Pr [ f(s(\mathcal{D}), x )=l) ]$ is the probability that the learned base model predicts label $l$ for $x$. The ensemble classifier $h$ predicts the label with the largest label probability for $x$, as:  
\begin{align}
    h(\mathcal{D}, x) = \arg\max_{l \in \{0,1\} } p_l
    \label{f(D,x)}
\end{align}

To prove the certified robustness of the mechanism $h(\mathcal{D}, x)$ against XBA, we need to find certified poisoning training size, which is the minimum number of poisoning training data  such that the ensemble classifier changes the prediction  for  $x$. Formally, we find the minimum  $r_{\mathcal{D}}$ such that the following inequality is satisfied for $\forall \mathcal{D}_+ \in  B(\mathcal{D},r)$:
\begin{equation}
    h(\mathcal{D}_+,x) \neq l \Leftrightarrow p'_l < p'_{\neg l}
\end{equation}
where $l \in \{0,1\}$ and $\neg l$ is the NOT operation of $l$ in the binary classification.

Finding exact values of $ p'_l$ and $p'_{\neg l}$ is difficult. Instead of that, we find the lower bound of $ p'_l$  and  upper bound of  $p'_{\neg l}$ ($l$ is the true label of $x$), we construct regions in the space $\Omega$, which is the joint space of  $X=s(\mathcal{D})$ and $Y = s(\mathcal{D}_+)$,  satisfying the conditions of the Neyman-Pearson Lemma \citep{neyman1933ix}. This enables us to derive the lower and upper bounds in these regions. Suppose we have a lower bound $\underline{p}_l$ of the largest label probability $p_l$ and an upper bound $\overline{p}_{\neg l}$ of the second largest label probability $p_{\neg l}$ when the classifier is trained on the clean training dataset. Formally,  $\underline{p}_l$ and $\overline{p}_{\neg l}$ satisfy:
\begin{equation}
   p_l \ge \underline{p}_l \ge \overline{p}_{\neg l} \ge p_{\neg l}
   \label{upper lower bounds}
\end{equation}

Following Theorem 1 in \citep{jia2020intrinsic}, we can have the following certified poisoning training size $r_{\mathcal{D}}$ as follows:

\textbf{Certified poisoning training size $r^*_{\mathcal{D}}$.}  Given a training dataset $\mathcal{D}$, a model $f(\cdot)$, and a testing sample $x$, the ensemble classifier $h$ is defined in Eq.~\ref{f(D,x)}. Suppose $l$ is the label with the largest probability predicted by $h$ for $x$ and $\neg l$ is the NOT operation of $l$ in the binary classification. We also have   $\underline{p}_l$ and $\overline{p}_{\neg l}$ satisfy Eq.~\ref{upper lower bounds}. The $h$ does not  predict the label $l$ for $x$ when the certified poisoning training size $r_{\mathcal{D}}$ is bounded by $r^*_{\mathcal{D}}$, i.e., we have: 
\begin{equation}
    h(\mathcal{D}_+, x) \neq l, \forall \mathcal{D}_+ \in B(\mathcal{D}, r^*_{\mathcal{D}} ), 
\end{equation}
where $r^*_{\mathcal{D}}$ is the solution to the following optimization problem:
\begin{equation}
     r^*_{\mathcal{D}}  = \arg \max_{ r_{\mathcal{D}}}  r_{\mathcal{D}}
     \label{r^*D}
\end{equation}
 s.t.
 \begin{align}
  \nonumber  & \max_{|\mathcal{D}|-r_{\mathcal{D}} \le |\mathcal{D}_+| \le |\mathcal{D}| + r_{\mathcal{D}}} (\frac{|\mathcal{D}_+|}{|\mathcal{D}|})^k - 2 \Big( \frac{\max(|\mathcal{D}, |\mathcal{D}_+|) -r_{\mathcal{D}} }{ |\mathcal{D}| } \Big)^k \\
    &+1 - (\underline{p}_l - \overline{p}_{\neg l} - \sigma_l - \sigma_{\neg l}) < 0
     \label{s.t.}
\end{align}

 where   $\sigma_l =\underline{p}_l- ( \lfloor{\underline{p}_l  n^k \rfloor} )/n^k $ and $\sigma_{\neg l} = \lceil{\overline{p}_{\neg l} n^k\rceil}/n^k -\overline{p}_{\neg l}$. 

Solving the problem in Eqs.~\ref{r^*D} and  in Eq.~\ref{s.t.} \citep{jia2020intrinsic}, we obtain the following results:
\begin{equation}
    r^*_{\mathcal{D}}  = \lceil{|\mathcal{D}| \Big( \sqrt[k]{1+(\underline{p}_l - \overline{p}_{\neg l} - \sigma_l - \sigma_{\neg l})}\Big) -1 \rceil}
\end{equation}

\textbf{Certified poisoning training size $r^*_{\mathcal{D}'}$.} Similar to find the certified poisoning training size $r^*_{\mathcal{D}}$, we obtain:

\begin{equation}
    r^*_{\mathcal{D}'}  = \lceil{|\mathcal{D}'| \Big( \sqrt[k]{1+(\underline{p}'_l - \overline{p}'_{\neg l} - \sigma'_l - \sigma'_{\neg l})}\Big) -1 \rceil}
\end{equation}
where  $\sigma'_l =\underline{p}'_l- ( \lfloor{\underline{p}'_l  n^{'k} \rfloor} )/n^{'k} $ and $\sigma'_{\neg l} = \lceil{\overline{p}'_{\neg l} n^{'k}\rceil}/n^{'k} -\overline{p}'_{\neg l}$. 

\textbf{Certified poisoning training size $r$.} 
 Intuitively, if $\mathcal{D}_o$ does not consist of poisoning data examples, then $r_\mathcal{D}$ is expected to be relatively the same with $r_{\mathcal{D}'}$. Otherwise, $r_{\mathcal{D}}$ can be significantly smaller than $r_{\mathcal{D}'}$ indicating that $\mathcal{D}_o$ is heavily poisoned with at least $r =   r_{\mathcal{D}} -  r_{\mathcal{D}'}  $ number of poisoning data samples towards opening backdoors on $x$. Therefore, after obtaining $r^*_{\mathcal{D}'}$ and $r^*_{\mathcal{D}}$, we have:

\begin{equation}
    r = r^*_{\mathcal{D}} - r^*_{\mathcal{D}'}
    \label{r-condition}
\end{equation}

\paragraph{Tightness of the certified poisoning training size $r$.} The bound in Theorem \ref{certified poisoning training theory} is tight and there is no existing smaller value of $r$ for the XBA to be successfully carried out. In our bound for the poisoning training size, the propriety data $\mathcal{D}$ cannot be changed by XBA; hence, $r^*_{\mathcal{D}}$ is fixed. In addition, $r_{\mathcal{D}'}^*$ is the minimum value that the prediction can be changed, so it is considered as a fixed value when the outsourced data $\mathcal{D}'$ remains the same, and it can be changed if the outsourced data is changed. For example, if  one more poisoning sample is added into $\mathcal{D}_o$, so $r_{\mathcal{D}'}^*$ becomes $r_{\mathcal{D}'}^*+1$, then the minimum value is also changed to be $r_{\mathcal{D}'}^*+1$. As a result, the certified robustness bound derived in Eq.~\ref{certified robustness bound} is tight.

\end{proof}

\section{Certified Robustness Bound using boosting randomized smoothing.}\label{app:pixeldp}
Directly applying the boosting RS (Appx. \ref{app:cert}) to \textsc{XRand} by \textit{only} using the noise $\alpha$ that is associated with $\Delta_{\alpha | w}$ may not be effective.
In the boosting RS, different base models (i.e., same structures and settings, just different random seeds when training), are trained on clean data; meanwhile, in \textsc{XRand}, the training data is backdoored data. As a result, the backdoored (and noisy) data in \textsc{XRand} along  with different base models in the boosting RS may increase  the variance of the outputs and then narrows down the certified robustness bound, which reduces the size of the bound. This results in a gap between the boosting RS and \textsc{XRand}.
 To close this gap, we add a smaller amount of noise, namely $\alpha_1$, into the input in addition to the noise $\alpha$ associated with $\Delta_{\alpha | w}$,
as follows:
\begin{align}
p_1 & := P\big(\bar{f}(x+\alpha + \alpha_1)=l\big)  \ge 1 - \frac{\sigma_{new}^2(M)}{ (c_l -c_{\neg l})^2}
    \label{p1 condition change}
\end{align}

Similarly, $\sigma^2_{new}$ is the variance of the classification margin, which is computed as in \cite{horvath2022boosting} associated with the new noise $(\alpha + \alpha_1)$.
Given the success probability $p_1$ in Eq.~\ref{p1 condition change} and the perturbation variance $\sigma_{\alpha + \alpha_1}^2$ (up to an incorrect prediction), the probability distribution over the certifiable radius $R$ is computed as follows:
\begin{align}
    P\Big(R = \sigma_{\alpha + \alpha_1} \Phi^{-1} \big( \underline{p_1} (N_1,N)  \big)  \Big) = \mathcal{B} (N_1,N,p_1) 
\end{align}
where $\mathcal{B} (N_1,N,p_1)$ is the probability of drawing $N_1$ successes in $N$ trials from a Binomial distribution with success probability $p_1$ and $\underline{p_1} (N_1,N) $ is the lower bound of the success probability of a  
Bernoulli experiment given $N_1$ successes in $N$ trials. 
The certified robustness bound is $R^* = \arg\max R$ such that the robustness condition in Eq.~\ref{p1 condition change} is satisfied.

\begin{figure}[t]
    \centering
    \includegraphics[width=1\linewidth]{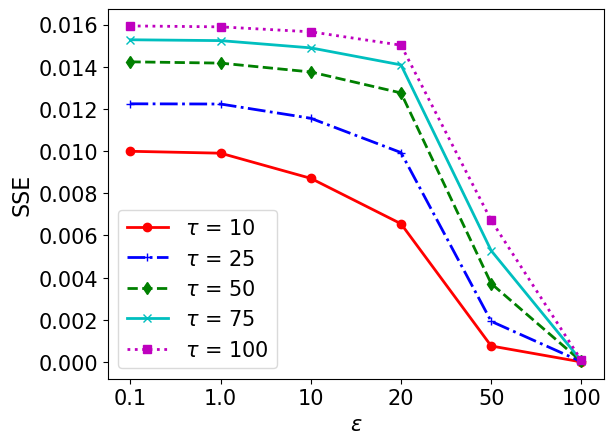}
    \caption{SSE values as a function of $\epsilon$ and $\tau$}
    \label{sse-eps-tau}
\end{figure}

\section{Experimental settings and results} \label{app:exp}

\paragraph{Platform.} Our experiments in this paper are implemented using Python 3.8 and conducted on a single GPU-assisted compute node that is installed with a Linux 64-bit operating system. The allocated resources include 8 CPU cores (AMD EPYC 7742 model) with 2 threads per core, and 100GB of RAM. The node is also equipped with 8 GPUs (NVIDIA DGX A100 SuperPod model), with 80GB of memory per GPU.

\paragraph{Dataset.} We conduct the XBA on malware classifiers against \textsc{XRand} explanations on three malware datasets. EMBER \citep{2018arXiv180404637A} is a representative benchmark dataset containing malicious and benign samples used for training malware classifiers. It consists of 2,351-dimensional feature vectors extracted from 1.1M Portable Executable (PE) files. The dataset contains $600,000$ training samples equally split between goodware and malware, and $200,000$ test samples, with the same class balance. 

We also test with malicious PDF files using the Contagio PDF dataset \cite{smutz2012malicious} which contains 10,000 PDF files equally split between goodware and malware, each sample is represented by a 135-dimensional feature vector. Finally, we evaluate our work on the Drebin dataset \cite{kumar2018effective} consisting of 5,560 malware and 123,453 goodware Android apps. Each sample contains 545,333 features extracted from the applications.

As mentioned in our threat model, we do not restrict the set of features that can be used by the attacker as backdoor triggers, so that our defense can be assessed against the strongest adversary.

\paragraph{Models.} For the EMBER dataset, we train two classifiers: LightGBM and EmberNN. LightGBM is a gradient boosting model released together with the EMBER dataset. It achieves good performance for malware binary classification with 98.61\% accuracy. Following Anderson et al. \citep{2018arXiv180404637A}, we use default parameters for training LightGBM (100 trees and 31 leaves per tree). EmberNN composes of four fully connected layers, in which the first three layers use ReLU activation functions, and the last layer uses a sigmoid activation function \citep{severi2021explanation}. EmberNN attains 99.14\% accuracy. 

\paragraph{Experimental results on Contagio and Drebin.}  We conduct additional experiments on the Contagio PDF \cite{smutz2012malicious} and Drebin \cite{kumar2018effective} datasets. The Contagio PDF dataset contains 10,000 PDF files equally split between goodware and malware, each sample is represented by a 135-dimensional feature vector. The Drebin dataset consists of 5,560 malware and 123,453 goodware Android apps. Each sample contains 545,333 features extracted from the applications. We use a Random Forest classifier for Contagio and a Linear Support Vector Machine for Drebin, and fix the trigger size to 30 features. The classifiers are released by (Severi et al, 2021) and we keep the same experimental settings. Figure \ref{fig:attpoison-contagio-debrin} shows the attack success rate of XBA using  the explanations returned by \textsc{XRand}.

On the Contagio and Drebin datasets, we observe similar behavior to our experiments with the EMBER dataset. The attack success rate decreases as we tighten the privacy budget, since the attacker has less access to the desired features. At 1\% poison rate and $\varepsilon = 10.0$, \textsc{XRand} can maintain a low success rate of 9.6\% in the Contagio dataset (Fig. \ref{fig:contagio}) and 7\% in the Drebin dataset (Fig. \ref{fig:debrin}).

\begin{figure}
    \centering
	\subfloat[Contagio]{
	    \includegraphics[width=0.45\linewidth]{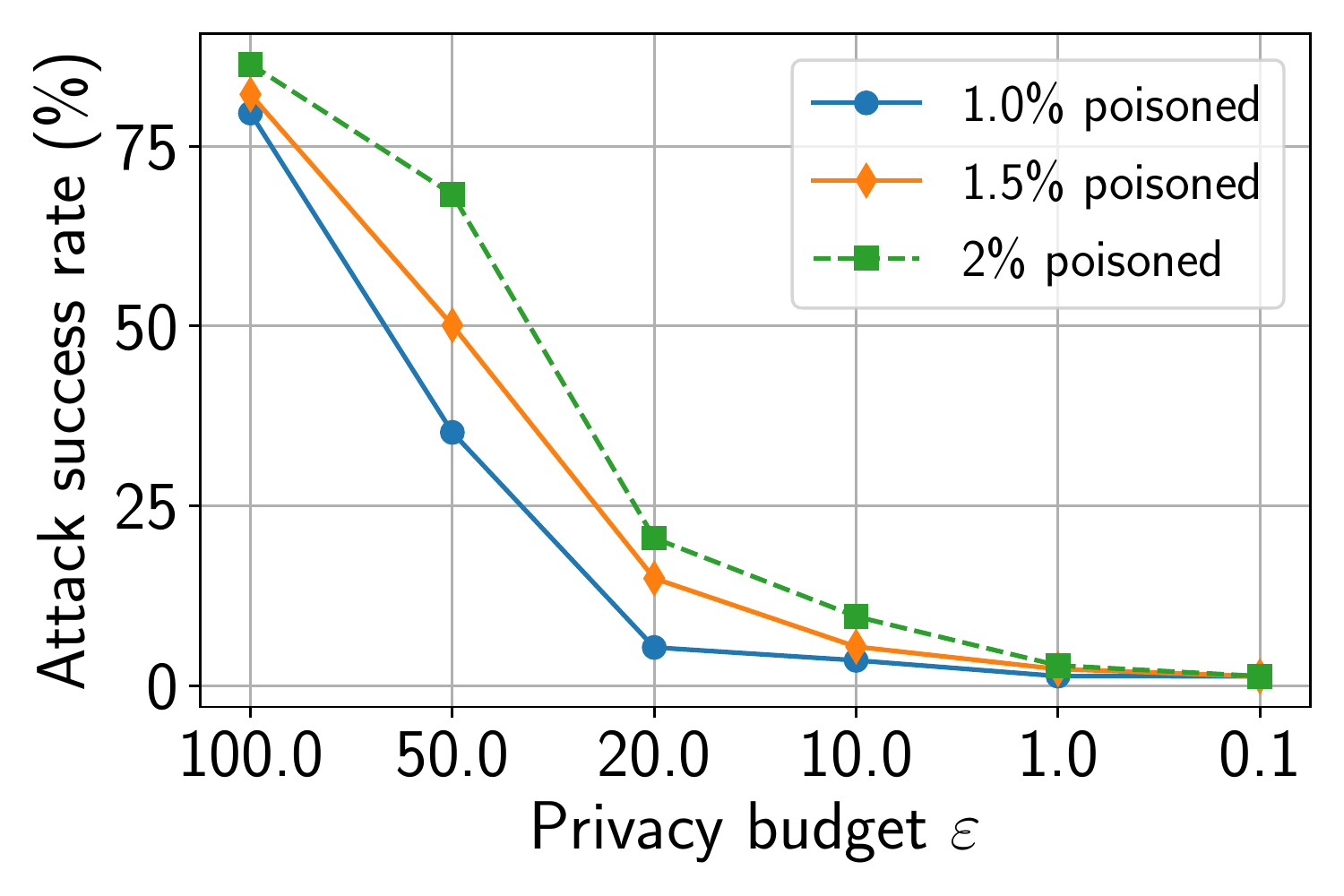}
    	\label{fig:contagio}
	}%
	\hfill
	\subfloat[Drebin]{
	    \includegraphics[width=0.45\linewidth]{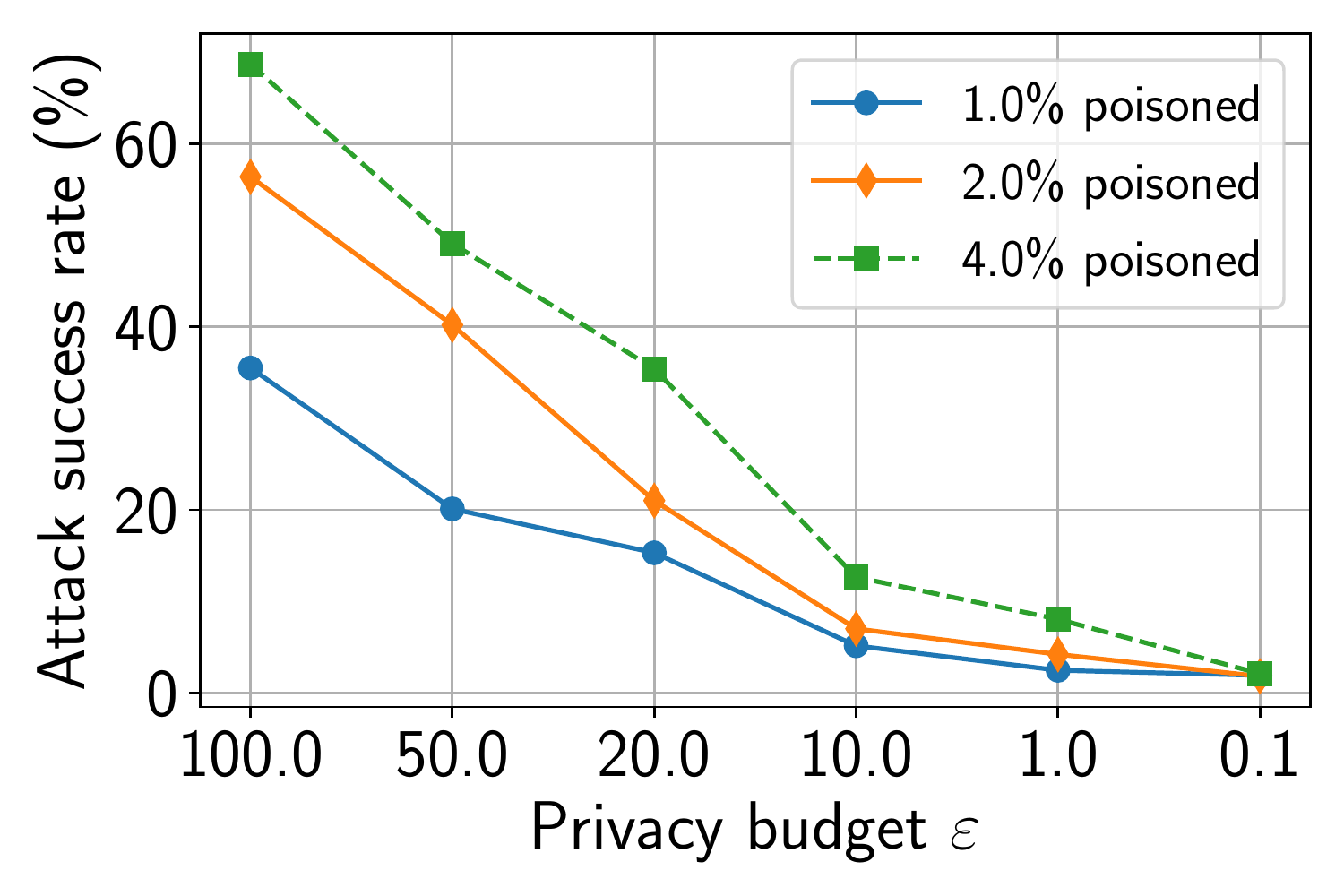}
	    \label{fig:debrin}
    }
    \caption{Attack success rate as a function of privacy budget $\varepsilon$ and the portion of poisoned samples on the Contagio PDF and Drebin datasets. The trigger size is fixed at 30.}
    \label{fig:attpoison-contagio-debrin}
\end{figure}

\begin{figure}
    \centering
    \includegraphics[width=0.8\linewidth]{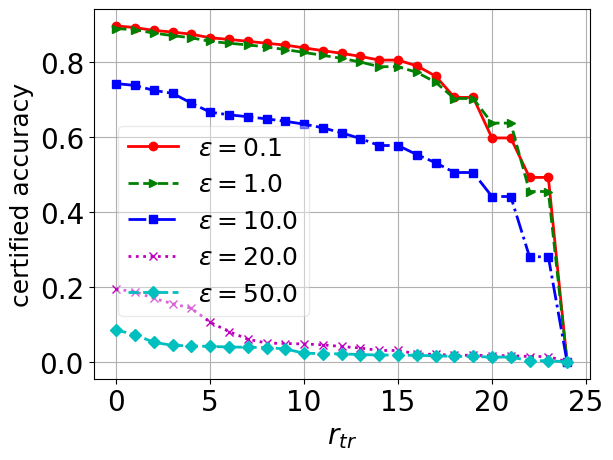}
    \caption{Certified robustness at the training time. The smaller privacy budget $\varepsilon$, the higher certified accuracy.}
    \label{fig:certified_bound_training}
\end{figure}

\paragraph{Certified Robustness.} To evaluate the certified robustness mechanism, as in \citep{lecuyer2019certified,jia2020intrinsic} we use the following metric:

{\small
\begin{equation}
 \text{certified accuracy} = \sum_{n=1}^{|test|}\frac{isCorrect(\mathcal{X}_n)~\&~isRobust(\mathcal{X}_n)}{|test|}
\label{certified accuracy}
\end{equation}
}

\noindent where $|test|$ is the number of testing samples, $isCorrect(\cdot)$ returns $1$ if the model makes a correct prediction (else, returns $0$), and $isRobust(\cdot)$ returns $1$ if the robustness size is larger than a given attack size $L$ (else, returns $0$).
When running with PixelDP and Boosting RS, we adopt the hyperparameter settings from the papers \cite{lecuyer2019certified} and \cite{horvath2022boosting}, respectively.

To verify the certified robustness at the training time, we conduct experiments with a wide range of $\varepsilon \in[0.1, 100.0]$. In this setting, we create $6,000$ poisoned samples by adding the trigger that follows the two-step LDP-preserving mechanism. From Fig. \ref{fig:certified_bound_training}, we observe that the smaller privacy budget $\varepsilon$, the higher certified accuracy. In fact, the smaller $\varepsilon$ imposes more noise which provides a better LDP guarantee and the model trained with noisier data is more robust against the backdoor attacks; hence, resulting in higher certified accuracy. It is obvious that the certified accuracy decreases when the threshold number of poisoning training samples $r_{tr}$ increases, since the term $isRobust(\cdot)$ decreases. 
For the inference-time bound, with the boosting RS certified robustness, we also conduct experiments with a wide range of privacy budget $\varepsilon \in[0.1, 50.0]$, the noise level $\sigma_{\alpha+\alpha_1} \in [0.25, 1.0]$, and the radius $r \in [0.25, 2.0]$. We consider $M = 5$ base classifiers in the ensemble model of the boosting RS. 
 With PixelDP, we conduct experiments given tight privacy budgets, i.e., $\varepsilon \in [0.1, 1.0]$. We obtain high  certified accuracy for these tight privacy budgets, i.e.,  $89.17\%$ and $90.42\%$ for $\varepsilon \in [0.1, 1.0]$, compared with $50\%$ obtained with the boosting RS. Here, we are dealing with a binary classification problem; therefore, the variance of the  model output is smaller than a multi-class classification problem.  That potentially  affects the effectiveness of the boosting RS.

\section{Visualizing \textsc{XRand}} \label{app:viz}

Figures \ref{fig:expcmp0}, \ref{fig:expcmp1}, and \ref{fig:expcmp4} visualize the explanations of some test samples before and after applying \textsc{XRand}. For the most part, the explanations from SHAP and \textsc{XRand} largely resemble one another.

\begin{figure*}
    \centering
	\subfloat[SHAP explanation (without \textsc{XRand})]{
	    \includegraphics[width=\linewidth]{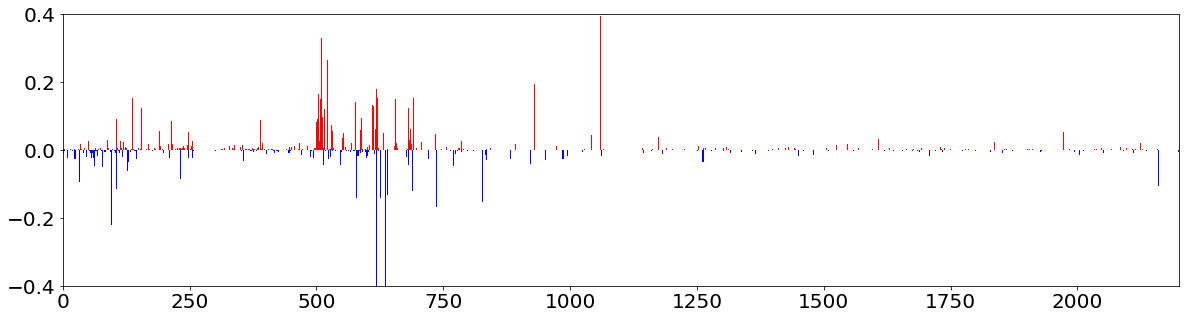}
	}%
	\hfill
	\subfloat[\textsc{XRand} explanation]{
	    \includegraphics[width=\linewidth]{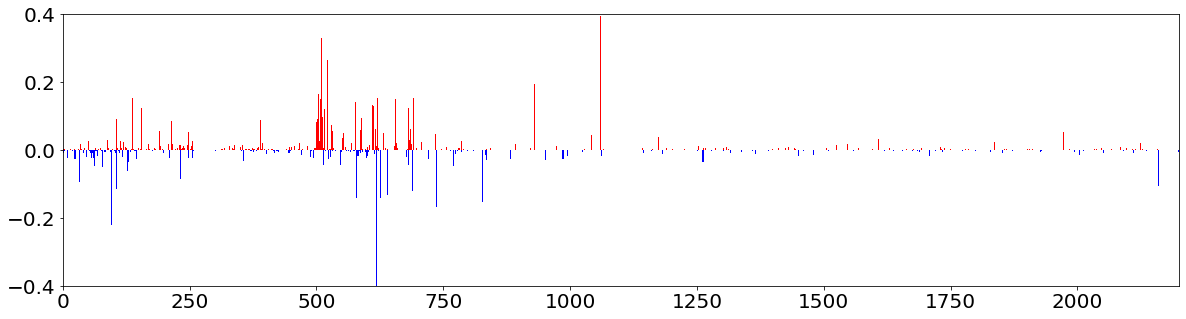}
    }
    \caption{Visualizing the SHAP explanation and our \textsc{XRand} explanation of test sample 1. The plot shows the SHAP value of each feature. The red vertical lines represent positive SHAP values that indicate malware-oriented features, while the blue vertical lines are negative SHAP values indicating goodware-oriented features.}
    \label{fig:expcmp0}
\end{figure*}

\begin{figure*}
    \centering
	\subfloat[SHAP explanation (without \textsc{XRand})]{
	    \includegraphics[width=\linewidth]{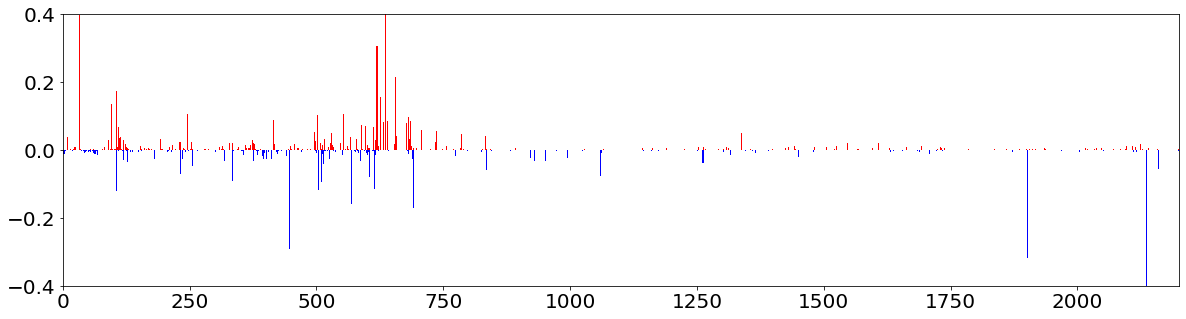}
	}%
	\hfill
	\subfloat[\textsc{XRand} explanation]{
	    \includegraphics[width=\linewidth]{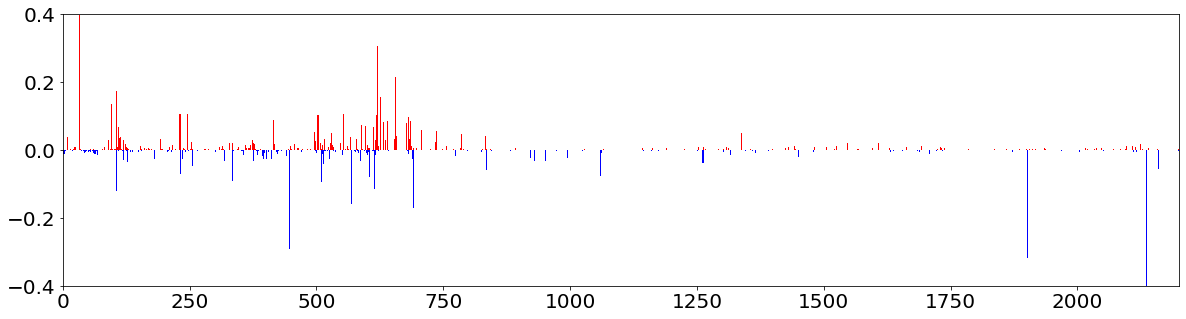}
    }
    \caption{Visualizing the SHAP explanation and our \textsc{XRand} explanation of test sample 2.}
    \label{fig:expcmp1}
\end{figure*}



\begin{figure*}
    \centering
	\subfloat[SHAP explanation (without \textsc{XRand})]{
	    \includegraphics[width=\linewidth]{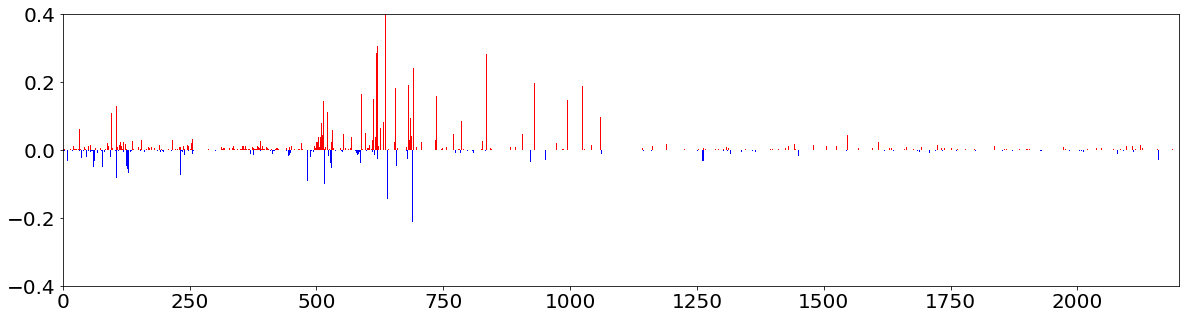}
	}%
	\hfill
	\subfloat[\textsc{XRand} explanation]{
	    \includegraphics[width=\linewidth]{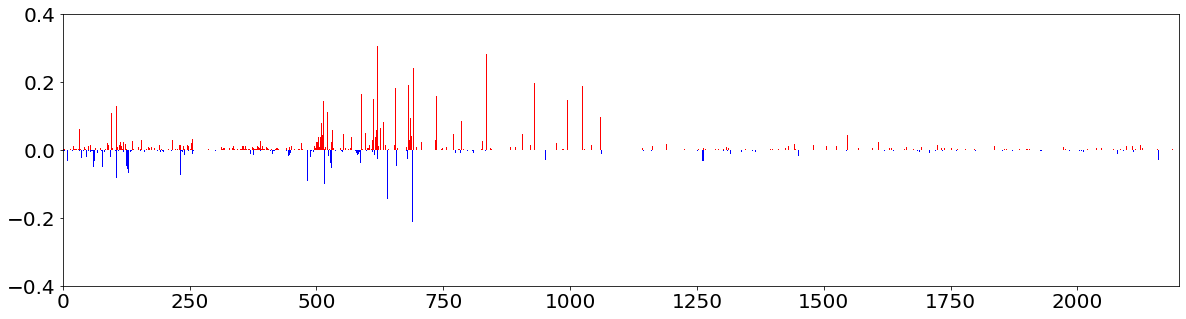}
    }
    \caption{Visualizing the SHAP explanation and our \textsc{XRand} explanation of test sample 3.}
    \label{fig:expcmp4}
\end{figure*}

\end{document}